\documentclass{article}

\usepackage[table]{xcolor}

\usepackage[utf8]{inputenc} 
\usepackage[T1]{fontenc}    
\usepackage{microtype}
\usepackage{caption}
\usepackage{subfigure}
\usepackage{booktabs} 
\usepackage{url}            
\usepackage{nicefrac}       
\usepackage{xspace}
\usepackage{amsthm}
\usepackage{amsmath}
\usepackage{amssymb}
\usepackage{amsfonts}
\usepackage{dsfont}
\usepackage{stmaryrd}
\usepackage{graphicx}
\usepackage{pdfpages}
\usepackage{enumitem}

\usepackage{hyperref}

\newcommand{\btheta}{\bm{\theta}}

\newcommand{\bz}{\bm{z}}

\newcommand{\bomega}{\bm{\omega}}

\newcommand{\grad}[2]{\nabla_{#1}#2}

\newcommand{\marginal}{\ensuremath{\bm{m}}\xspace}

\definecolor{ourspecialtextcolor}{rgb}{0.528, 0.471, 0.701}
\newcommand{\latentdist}{p_{\btheta}(\bz \mid \bm{A} \bm{z} = \bm{k} )}

\newcommand{\vv}[1]{\boldsymbol{#1}}

\newcommand{\lone}{L1\xspace}
\newcommand{\ltwo}{L2\xspace}

\newcommand{\x}{\ensuremath{\vv{x}}\xspace}

\newcommand{\y}{\ensuremath{\vv{y}}\xspace}
\newcommand{\yhat}{\ensuremath{\hat{\vv{y}}}\xspace}
\newcommand{\z}{\ensuremath{\vv{z}}\xspace}
\newcommand{\X}{\ensuremath{\mathcal{X}}\xspace}
\newcommand{\Y}{\ensuremath{\mathcal{Y}}\xspace}
\newcommand{\Z}{\ensuremath{\mathcal{Z}}\xspace}

\newcommand{\Logits}{\ensuremath{\Theta}\xspace}

\newcommand{\E}{\ensuremath{\mathbb{E}}\xspace}
\newcommand{\R}{\ensuremath{\mathbb{R}}\xspace}

\newcommand{\Loss}{\ensuremath{L}\xspace}
\newcommand{\loss}{\ensuremath{\ell}\xspace}

\newcommand{\logits}{\ensuremath{\vv{\theta}}\xspace}

\newcommand{\vparam}{\ensuremath{\vv{v}}\xspace}
\newcommand{\uparam}{\ensuremath{\vv{u}}\xspace}
\newcommand{\params}{\ensuremath{\vv{\omega}}\xspace}
\newcommand{\encoder}{\ensuremath{h_{\vparam}}\xspace}
\newcommand{\decoder}{\ensuremath{f_{\uparam}}\xspace}

\newcommand{\unnormp}{\ensuremath{p_{\btheta}\parentheses{\boldsymbol{z}}\xspace}}

\newcommand{\subsetp}{\ensuremath{ p_{\btheta}\parentheses{\boldsymbol{z} \mid \textstyle \bm{A} \bm{z} = \bm{k}}\xspace}}

\newcommand{\wrt}{w.r.t.\ }
\newcommand{\equality}[1]{\bm{A} \bm{#1} = \bm{k}}

\newcommand{\guy}[1]{}
\newcommand{\zz}[1]{}

\usepackage[accepted]{icml2025}

\usepackage{amsmath}
\usepackage{amssymb}
\usepackage{mathtools}
\usepackage{amsthm}

\usepackage{bm}

\usepackage[capitalize,noabbrev]{cleveref}

\theoremstyle{plain}
\newtheorem{theorem}{Theorem}[section]
\newtheorem{proposition}[theorem]{Proposition}
\newtheorem{lemma}[theorem]{Lemma}

\theoremstyle{definition}

\theoremstyle{remark}

\usepackage[textsize=tiny]{todonotes}

\usepackage{mathtools}
\DeclarePairedDelimiterX{\infdivx}[2]{(}{)}{%
  #1\;\delimsize|\delimsize|\;#2%
}
\DeclarePairedDelimiter\parentheses{\lparen}{\rparen}

\usepackage{multirow}
\usepackage{adjustbox}

\DeclarePairedDelimiter\floor{\lfloor}{\rfloor}

\begin{document}

\twocolumn[
\icmltitle{Deep Generative Models with Hard Linear Equality Constraints
}

\begin{icmlauthorlist}
\icmlauthor{Ruoyan Li}{UCLA Math}
\icmlauthor{Dipti Ranjan Sahu}{UCLA CS}
\icmlauthor{Guy Van den Broeck}{UCLA CS}
\icmlauthor{Zhe Zeng}{NYU CS}
\end{icmlauthorlist}

\icmlaffiliation{UCLA Math}{Department of Mathematics, University of California, Los Angeles}
\icmlaffiliation{UCLA CS}{Computer Science Department, University of California, Los Angeles}
\icmlaffiliation{NYU CS}{Computer Science Department, New York University}

\icmlcorrespondingauthor{Ruoyan Li}{liruoyan2002@g.ucla.edu}

\icmlkeywords{Machine Learning, ICML}

\vskip 0.3in
]

\printAffiliationsAndNotice{}  

\begin{abstract}
While deep generative models~(DGMs) have demonstrated remarkable success in capturing complex data distributions, they consistently fail to learn constraints that encode domain knowledge and thus require constraint integration. 
Existing solutions to this challenge have primarily relied on heuristic methods and often ignore the underlying data distribution, harming the generative performance.
In this work, we propose a probabilistically sound approach for enforcing the hard constraints into DGMs 
to generate constraint-compliant and realistic data.
This is achieved by our proposed gradient estimators that allow the constrained distribution, the data distribution conditioned on constraints, to be differentiably learned.
We carry out extensive experiments with various DGM model architectures over five image datasets and three scientific applications in which domain knowledge is governed by linear equality constraints.
We validate that the standard DGMs almost surely generate data violating the constraints.
Among all the constraint integration strategies,
ours not only guarantees the satisfaction of constraints in generation but also archives superior generative performance than the other methods across every benchmark.
\end{abstract}

\section{Introduction}

\begin{figure}[t]
    \centering
    \includegraphics[width=0.95\columnwidth]{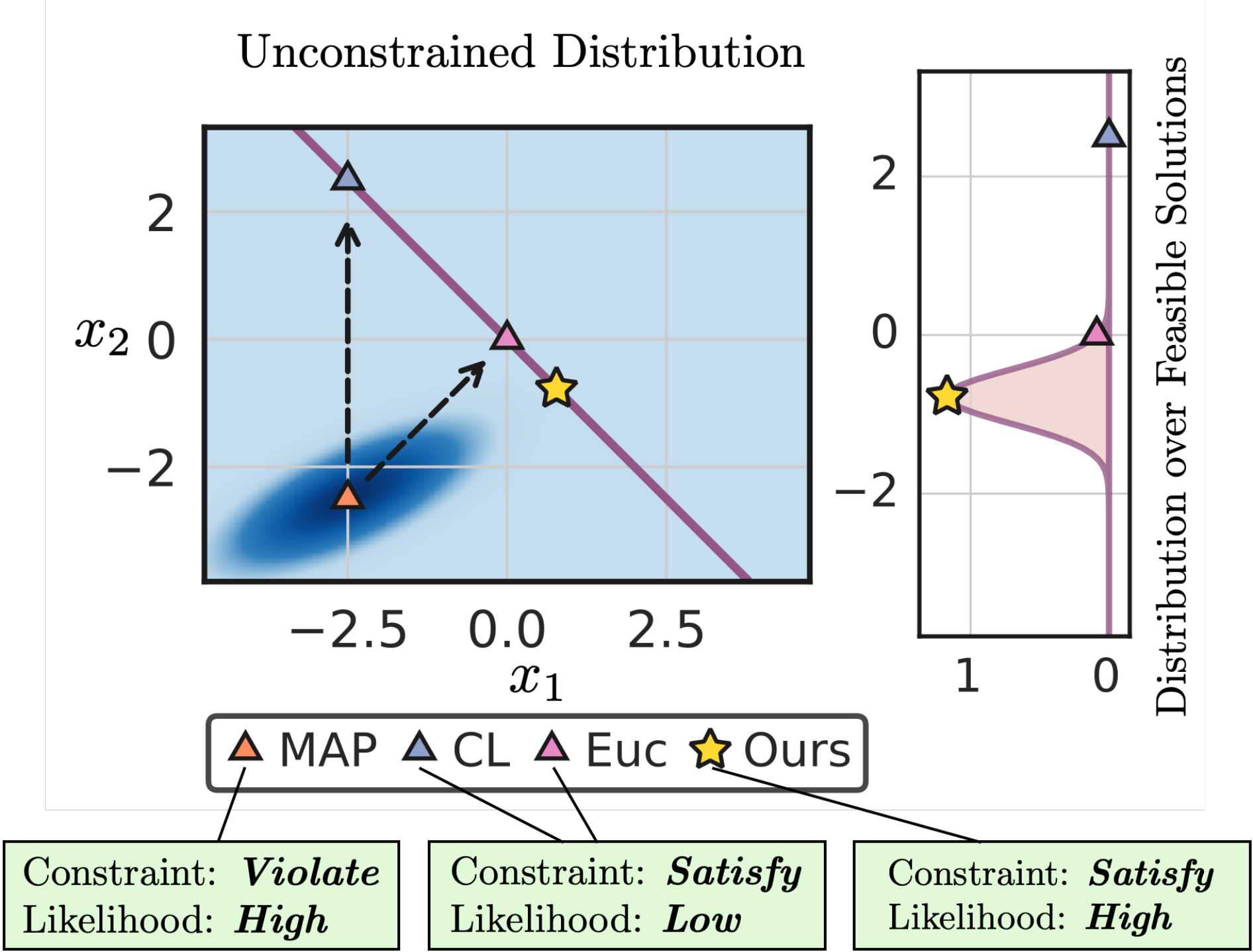}
    \caption{Comparison of different methods for generating samples that satisfy linear equality constraints. The left panel shows the original unconstrained distribution in a $2$-dimensional plane, with the purple line representing the constraint $x_1 + x_2 = 0$. 
    Our proposed method generates the most realistic sample as indicated by the right figure,
    outperforming existing methods that optimize for L1 distance (CL) and L2 distance (Euc).
    }\label{fig:intro}
\end{figure}

Deep generative models (DGMs) have made great progress in generating realistic data by capturing the underlying patterns and distributions of a dataset. Simultaneously, researchers have leveraged
machine-learning techniques to accelerate scientific discoveries and simulate complex systems. However, they have been found to struggle with learning the domain knowledge~\citep{zhang2023paradox}.
For example, if a chemist trains a model on a charge-neutral molecule dataset for predicting charges of each atom in a given molecule, they  want the predicted charges to sum up to zero, satisfying the charge neutrality property as background knowledge. This is a true concern for chemists in \citet{raza2020message} as they find state-of-the-art models almost surely generate predictions that violate this property and thus are useless for downstream tasks. 
The ability to incorporate domain knowledge into generative modeling remains crucial for the broad application of AI, particularly in scientific domains.

Domain knowledge such as the one shown above takes the form of \emph{linear-equality constraints}, an important class of constraints that have been studied by many given their wide applications. Another example is the mass balance and stoichiometry in chemical engineering whose processes are governed by linear equality constraints~\citep{chen2024hard}.
Such constraints are also necessary for stock investment allocation in financial engineering~\citep{zhang2020deep, butler2021integrating}.
While this list can be made longer, it is surprisingly challenging to enforce these seemingly simple constraints into DGMs,
limiting their application in scenarios where violations of constraints are unacceptable.

\begin{figure*}[t]
    \centering
    \includegraphics[width=0.85\textwidth]{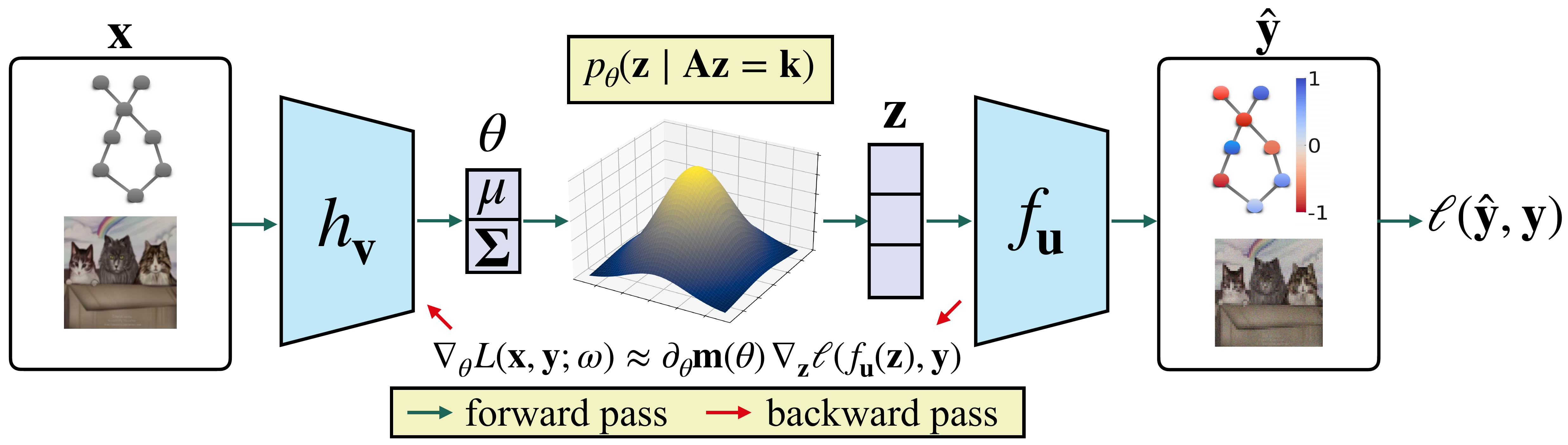}
    \caption{
    The constrained model considered in this work.
    It involves an encoder  $\encoder$ that outputs $\btheta$ to parameterize a latent distribution constrained by the linear equality constraint $\bm{A} \bm{z} = \bm{k}$. We first study when the objective admits a closed-form expression such that standard training is amenable. 
    We further propose and study various gradient estimators for the general case
    by combining exact sampling in the forward pass and gradient approximations in the backward pass.
    }
    \label{fig:problem-setting}
\end{figure*}

Incorporating constraints directly into the differentiable learning process can improve both generalization and data efficiency, compared to imposing them only in a post-processing step.
Existing methods for enforcing such constraints more or less share the same philosophy:
first generating a candidate prediction by 
sampling from the unconstrained distribution, 
which is almost certainly to violate the constraints,
and then making \emph{minimal changes} to the candidate such that it satisfies the constraint and is returned as the final prediction.
They differ in the \emph{heuristics} proposed for deriving minimal changes and accordingly the \emph{gradient estimators} for differentiating through the samples.
For example, given a candidate sample $\bm{x}$, \citet{stoian2024realistic} propose a 
\textit{Constraint Layer} (CL) that incrementally updates $i$-th feature $\bm{x}_i$ to return a sample $\tilde{\bm{x}}$ that satisfies the constraints and minimize the $\lone$ distance between ${\bm{x}}$ and $\tilde{\bm{x}}$ while \citet{chen2024hard} (Euc) propose to use $\ltwo$ distance and formulate such projection as quadratic programming problems.
We visualize these baselines works at deployment in Figure~\ref{fig:intro} where the arrows indicate the minimal changes to a given maximum a posteriori (MAP) solution as a candidate.
We observe that while these predictions are obtained by making minimal perturbations to the MAP solution, they result in low-likelihood constrained samples.
This is because these transformations only consider sample distances but disregard the underlying data distribution learned by DGMs.

In this work, we present a probabilistically sound framework to incorporate the hard linear equality constraints into the DGMs based on a simple yet effective idea: instead of \emph{transforming the sample}, we propose to \emph{transform the distribution}.
While the baseline methods consider first sampling and then constrain the samples,
we propose to first constrain the distribution to the feasible space satisfying the constraints, and then sample from the constrained distribution.
For example, in Figure~\ref{fig:intro}, we generate our sample by first deriving the distribution on the right figure and then taking the MAP of the constrained distribution.
The benefits of this approach are two-fold:
the resulting samples are guaranteed to satisfy the constraints and meanwhile, they closely resemble the true data with high likelihood.

Specifically, we make the following \textbf{contributions}.
(i) We propose and compare several gradient estimators that allow the end-to-end training of such constrained distributions. We further validate that one specific design of gradient estimator is significantly more effective than the others.
(ii) We demonstrate that our approach is flexible in two key ways: it is agnostic to the DGM architectures, making it applicable across a wide range of models; also, it allows constraints to be seamlessly integrated into any layer of the DGM.
(iii) To conduct extensive empirical evaluations, we compare the different constraint enforcement approaches across five image datasets--where the brightness is governed by the constraints--and three scientific applications in which domain knowledge takes the form of linear constraints,
involving multiple DGM variants from different model classes: VAEs, diffusion models and graph neural networks. 
(iv)
We show that standard DGMs always fail to learn the constraint, underscoring the need for explicit constraint enforcement to generate realistic samples compliant with the domain knowledge.
(v)
We further demonstrate that our approach consistently achieves better generative performance than all baseline methods across every benchmark.
\emph{Overall, our approach paves the way for a principled design of integrating constraints into DGMs, enabling the generation of constraint-compliant, realistic samples.}

\section{Related Work}
Our work lies in the field of neuro-symbolic AI where integrating constraints as background knowledge into deep learning models is widely studied~\citep{garcez2023neurosymbolic}. 

\textbf{Constraint Enforcement.}
There are multiple ways to enforce constraints in the neural networks.
Some directly incorporate background knowledge as network layers~\citep{ahmed2022semantic, giunchiglia2021multi}.
In the context of generative tasks, \citet{di2020efficient} embed propositional logic constraints into Generative Adversarial Networks (GANs) for structured object generation, while \citet{misino2022vael} integrate probabilistic logic programming \citep{de2007problog} with Variational Autoencoders (VAEs). 
\citet{Hendriks2020LinearlyCN} impose linear operator constraints within the architecture of feedforward neural networks, but their approach does not generalize to other model architectures. 
\citet{WangICML23} enforce positive linear constraints using a Sinkhorn algorithm, where their method is only applicable to output variables being unit hypercubes. 
\citet{chen2024hard} formulate constraint satisfaction as an optimization problem. They use the $\ltwo$ distance and derive projections based on the Karush–Kuhn–Tucker (KKT) conditions. 
\citet{amos2017optnet} and \citet{donti2017task} integrate quadratic programming solvers as differentiable modules within end-to-end trainable deep networks
while some other works formulate it as submodular optimization problems~\citet{Djolonga2017}, \citet{Tschiatschek2018DifferentiableSM}, and \citet{wilder2019melding}. 
Most of these approaches ignore underlying data distributions when solving optimization problems
while in our method we leverage the information from the constrained distribution for optimizing the DGMs.
Quantitative comparisons between our method and existing work are further presented in the experimental section.

\textbf{Constrained Sampling.}
There are some existing works in the field of statistical modeling that study how to sample from linearly constrained Gaussian distributions that can be potentially applied as post-processing steps.
For example, \citet{risks6030064} investigates a linear weighted constraint for independent standard Gaussian variables, while \citet{LAMBONI2022199} focuses on a fixed-sum constraints for independent Gaussian variables with zero means.
However, these methods are not differentiable and thus cannot be incorporated into the training of DGMs, 
leading to low data efficiency.

\textbf{Exactly-k Constraints.}
The discrete counterpart of linear equality constraints, 
the exactly-k constraints defined as $\sum_i x_i = k$ with $x_i$ being categorical variables is studied by many.
\citet{maddison2016concrete} and \citet{jang2016categorical} propose similar ideas to refactor the non-differentiable sample from a categorical distribution with a differentiable sample from Gumbel-Softmax distributions. 
Other gradient estimators for this constraint
either employ variants of score function and straight-through estimator or propose certain relaxations~\citep{kim2016exact,chen2018learning,grover2019stochastic,Sang2019reparameterizable}. 
Closely related to our work is
a recently introduced gradient estimator~\citep{ahmed2022simple} that leverages the
constrained marginal distribution as an informative proxy for differentiation.

\textbf{Soft Constraints.}
A line of research integrates the constraints by optimizing for the probability of constraint satisfaction, encouraging the model to generate compliant samples~\citep{diligenti2012bridging,xu2018semantic,fischer2019dl2,badreddine2022logic,stoian2024exploiting,shukla2024unified}.
This is achieved by modifying the loss function with differentiable constraint probabilities.
However, these methods do not guarantee the satisfaction of the constraint.

\section{Problem Statement}
\label{sec:Problem Statement}
Instead of sampling from unconstrained DGM $\unnormp$ using heuristics,
our goal is to 
constrain the DGM
as below
\begin{equation}
\label{eqn:pipeline}
    \logits 
    = \encoder(\x), 
    \quad \z \sim 
    \subsetp, 
    \quad \yhat = \decoder(\z),
\end{equation}
where $\x \in \X$ and $\yhat \in \Y$ denote feature inputs and target outputs, respectively, $\encoder: \X \rightarrow
\Logits$ and $\decoder: \Z \rightarrow \Y$ are smooth, parameterized mappings.
Parameters $\btheta$ induce a Gaussian distribution $\mathcal{N} \left(\boldsymbol{\mu}, \boldsymbol{\Sigma} \right)$ over the latent variables $\z$ where parameters $\btheta = (\bm{\mu}, \bm{\Sigma})$ consist of
the mean vector $\boldsymbol{\mu} \in \R^{n}$ and the covariance matrix $\boldsymbol{\Sigma} \in \R^{n \times n}$.
That is, $\z$ has its probability density function (p.d.f.) defined as 
$\unnormp = \frac{1}{(2\pi)^{n/2} |\boldsymbol{\Sigma}_{\btheta}|^{1/2}} \exp\left(-\frac{1}{2} (\bm{z} - \boldsymbol{\mu}_{\btheta})^\top \boldsymbol{\Sigma}_{\btheta}^{-1} (\bm{z} - \boldsymbol{\mu}_{\btheta})\right)$. $\equality{z}$ denotes the linear equality constraints with $\bm{A} \in \R^{a \times n}$, $\text{rank} (\bm{A}) = a \leq n$, and $\bm{k} \in \R^{a}$, enforced over the DGM $\unnormp$ inducing a constrained distribution $\subsetp$.

This formulation is general and it subsumes various DGM classes that integrate the linear equality constraint in 
\textit{1) output} (when the mapping $\decoder$ is the identity function), where the goal of constrained generative modeling is to learn the parameters $\btheta$ such that the constrained  distribution 
approximates the underlying data distribution;
or \textit{2) latent space}, where the constrained generative modeling learns a constrained posterior distribution 
$p_{\btheta}$
over latent variables.
The training of this model is by optimizing an expected loss:
\begin{equation}
\label{eqn:Loss}
\begin{split}
    &\Loss(\x, \y; \params) = \E_{\z \sim p_{\logits}(\z \mid \bm{A} \z = \bm{k})}[\loss(\decoder(\z), \y)]\\ 
    &\textit{with~~} \params = (\vparam, \uparam) \textit{~~and~~} \logits = \encoder(\x),
\end{split}
\end{equation}

where $\loss: \Y \times \Y \rightarrow \R^+$ is a point-wise loss function.
Figure~\ref{fig:problem-setting} shows
a visualization of the pipeline.

\section{Gradient Estimation for Linear Equality}

\begin{table*}[t]
    \centering
    \caption{Summary of gradient estimators. 
      The first block presents baseline estimators and the second block presents our proposed ones.
      In the forward pass, we sample exactly from the constrained distribution 
      (Proposition~\ref{prop: Gaussian Constrained Distribution}). 
      In the backward pass, we use $\marginal(\btheta)$ as a differentiable proxy. 
      For \textit{Constrained Layer} and \textit{Constrained Reparametrization}, 
      the samples for deriving the gradient estimations are generated using the unconstrained reparametrization trick.}
    \begin{adjustbox}{max width=0.95\textwidth}
    \begin{tabular}{l l l}
    \toprule
    \textbf{\textsc{Gradient Estimator}}
      & \textbf{\textsc{Proxy $\marginal(\btheta)$}}
      & \textbf{\textsc{Description}} \\
    \midrule
    \rowcolor{gray!15}
    Random
      & --
      & Sample a random gradient from $\mathcal{N}(\bm{0}, \bm{I})$. \\
    Unconstrained Marginal
      & $p_{\btheta}(z_i)$
      & p.d.f.\ of unconstrained $\z$ as a proxy for $\z$. \\
    \rowcolor{gray!15}
    Constrained Layer
      & $\mathrm{CL}\!\bigl(\bm{\mu} + \bm{\sigma}\!\odot\!\bm{\epsilon}\bigr)$
      & Use reparametrization trick as a proxy for $\z$. \\
    \rowcolor{gray!15}
      & 
      & Constrained Layer enforces $\bm{A}\z = \bm{k}$. \\
    \midrule
    \rowcolor{gray!15}
    Constrained Reparametrization
      & $\bm{\hat{z}} = \bm{\mu} + \bm{\sigma}\!\odot\!\bm{\epsilon}$
      & Apply variance-weighted correction. \\
    \rowcolor{gray!15}
      & $\displaystyle \z = \bm{\hat{z}}
         \;+\; \bm{\Sigma}\,\bm{A}\,\bigl(\bm{A}\,\bm{\Sigma}\,\bm{A}^T\bigr)^{-1}
         \bigl(\bm{k} - \bm{A}\,\bm{\hat{z}}\bigr)$
      & Ensures $\bm{A}\z = \bm{k}$. \\
    Constrained Marginal
      & $p_{\btheta}\!\bigl(z_i \mid \bm{A}\,\z = \bm{k}\bigr)$
      & p.d.f.\ of conditional marginals as a proxy for $\z$. \\
    \rowcolor{gray!15}
    Marginal Expectation
      & $\mathbb{E}_{\,z_i\sim p_{\btheta}(\,z_i\mid \bm{A}\z=\bm{k}\,)}[\,z_i\,]$
      & Expectation of conditional marginals as a proxy for $\z$. \\
    \bottomrule
    \end{tabular}
    \end{adjustbox}
    \label{tab:estimator summary}
\end{table*}

Standard auto-differentiation can not be directly applied to the expected loss due to two main obstacles. First, for the gradient of the expected loss $L$ \wrt parameters $\uparam$ in the decoder mapping $\decoder$, which is defined as 
\begin{equation}
    \scalebox{0.9}{$
        \nabla_{\uparam} \Loss(\x, \y; \params) = \E_{\z \sim \latentdist} \partial_{\uparam} \decoder(\z, \x)^\top \nabla_{\yhat} \loss(\yhat, \y)
    $}
\end{equation}
with $\hat{y} = \decoder(\z)$ being the decoding of a latent sample $\z$,
this expectation does not allow closed-form solution in general and requires Monte-Carlo estimations by sampling $\z$ from the constrained distribution $p_{\logits}(\z \mid \bm{A} \bm{z} = \bm{k})$.
Another issue arises in the gradient of $L$ \wrt parameters $\vparam$ in the encoder mapping which is defined as 
\begin{equation}\label{eqn:grad-L-enc}
    \nabla_{\vparam} \Loss(\x, \y; \params) = \partial_{\vparam} \encoder(\x)^\top \nabla_{\logits} \Loss(\x, \y; \params).
\end{equation}
The obstacle lies in the computation of the gradient of the expected loss $L$ \wrt $\btheta$
defined as $\nabla_{\logits} \Loss(\x, \y; \params) := \nabla_{\logits} \E_{\z \sim p_{\logits}(\z \mid \bm{A} \bm{z} = \bm{k})}[\loss(\decoder(\z, \x), \yhat)]$ which requires gradient estimators.
In this section, we tackle the gradient estimation for the linear equality constraint by solving the aforementioned two subproblems:
\textbf{(P1)} how to \textit{sample exactly} from the constrained distribution $p_{\btheta}(\z \mid \bm{A} \bm{z} = \bm{k})$ and \textbf{(P2)} how to \textit{estimate} $\nabla_{\logits} \Loss(\x, \y; \params)$. 
Solutions to these two subproblems, when combined, allow us to train the constrained models in an end-to-end manner. 
For \textbf{(P1)}, we observe that the constrained distribution $p_{\btheta}(\z \mid \bm{A} \bm{z} = \bm{k})$ is a multivariate Gaussian distribution and thus performing exact sampling is straightforward as long as we derive the parameters for the constrained distribution. We formally state this observation in Appendix~\ref{proof:SystemofLinearEqualityConstrainedDistribution}.
In the following, we present the various design choices as candidate solutions to \textbf{(P2)}.

\subsection{Gradient Estimator Design}

The reparameterization trick \citep{kingma2022autoencoding}
is perhaps the most commonly used technique for differentiating through random samples.
Specifically, it expresses a sample $\bm{z}$ as $\bm{z} = \bm{\mu} + \bm{\sigma} \odot \bm{\epsilon}$, where $\bm{\epsilon} \sim N(\bm{0}, \bm{I})$ and $\bm{\mu}$ and $\bm{\sigma}$ are mean and standard deviation, respectively.
However, when it is directly applied to the constrained DGMs, it simply ignores the constraint information:
even though the sample $\bm{z}$ is drawn from the feasible space that satisfies the constraints,
the addition of the random noise results in violation of the constraints and thus the derived gradient is for optimizing an unconstrained model.

Instead, we propose novel ways to build gradient estimators that are able to leverage the constraint information and effectively optimize the constrained models.
We first propose an approximation to the problematic term in Equation~\ref{eqn:grad-L-enc} as
\begin{equation}
    \grad{\btheta}{L(\x, \y; \bomega)}
    \approx \partial_{\btheta} \marginal(\btheta) \grad{\z}{\ell(\x, \y; \bomega)},
\end{equation}
where $\marginal(\btheta)$ should be chosen as a function that can be efficiently computed and differentiated and meanwhile encode constraint information.
Here, we consider two candidates for $\marginal(\btheta)$:
1) the conditional marginal probability density $p_{\btheta}( z_i| \bm{A} \z = \bm{k})$;
and 2) the expectation of $z_i$ under the conditional marginal, that is,
$\E_{z_i \sim p_{\btheta}(z_i | \bm{A} \z = \bm{k})}[z_i]$.
The intuition behind the adoption of these marginal distributions is that, by conditioning on the constraints, their gradients provide a differentiable proxy for optimizing the constrained distribution. It encourages the constrained model to generate constraint-compliant samples with low loss, allowing for efficient end-to-end training of DGMs.

While these two quantities seem to encode similar information,
we make an interesting observation in our empirical study that in continuous domains, the use of expectation is consistently more effective than conditional marginals.
We further provide a baseline estimator that chooses $\marginal(\btheta)$ to be the unconstrained marginals  $p_{\btheta}(z_i)$, meaning that the constraint is ignored during the training process; empirical results show that such ignorance can harm model performance even though the constraint is enforced at inference.

The remaining question is how to compute and differentiate $\marginal(\btheta)$ for these two different estimators.
We present below the theoretical results to show that constrained marginals and their expectations admit closed-form representation and thus allow efficient computations.

\begin{proposition}[Gaussian Conditional Marginal and Expectations]\label{prop: Gaussian Conditional Marginal}
    Given $\z = \left( z_1, \ldots, z_n \right)^T \sim \mathcal{N} \left( \bm{\mu}, \bm{\Sigma} \right)$, the conditional marginal $p_{\logits}(z_i \mid \bm{A} \z = \bm{k})$ follows a univariate Gaussian distribution with mean $\overline{\mu}_i = \mu_i + \bm{e}_i^T \bm{\Sigma} \bm{A} \left( \bm{A} \bm{\Sigma} \bm{A}^T \right)^{-1} \left( \bm{k} - \bm{A} \bm{\mu} \right)$ and variance $\overline{\sigma}_i^2 = \bm{e}_i^T \bm{\Sigma} \bm{e}_i - \bm{e}_i^T \bm{\Sigma} \bm{A}^T \left( \bm{A} \bm{\Sigma} \bm{A}^T \right)^{-1} \bm{A} \bm{\Sigma} \bm{e}_i$.
    Further, the expectation of the marginal distribution is $\overline{\mu}_i$.
\end{proposition}

\begin{figure}[t]
    \centering
    \begin{subfigure}
        \centering
        \includegraphics[height=0.125\textheight]{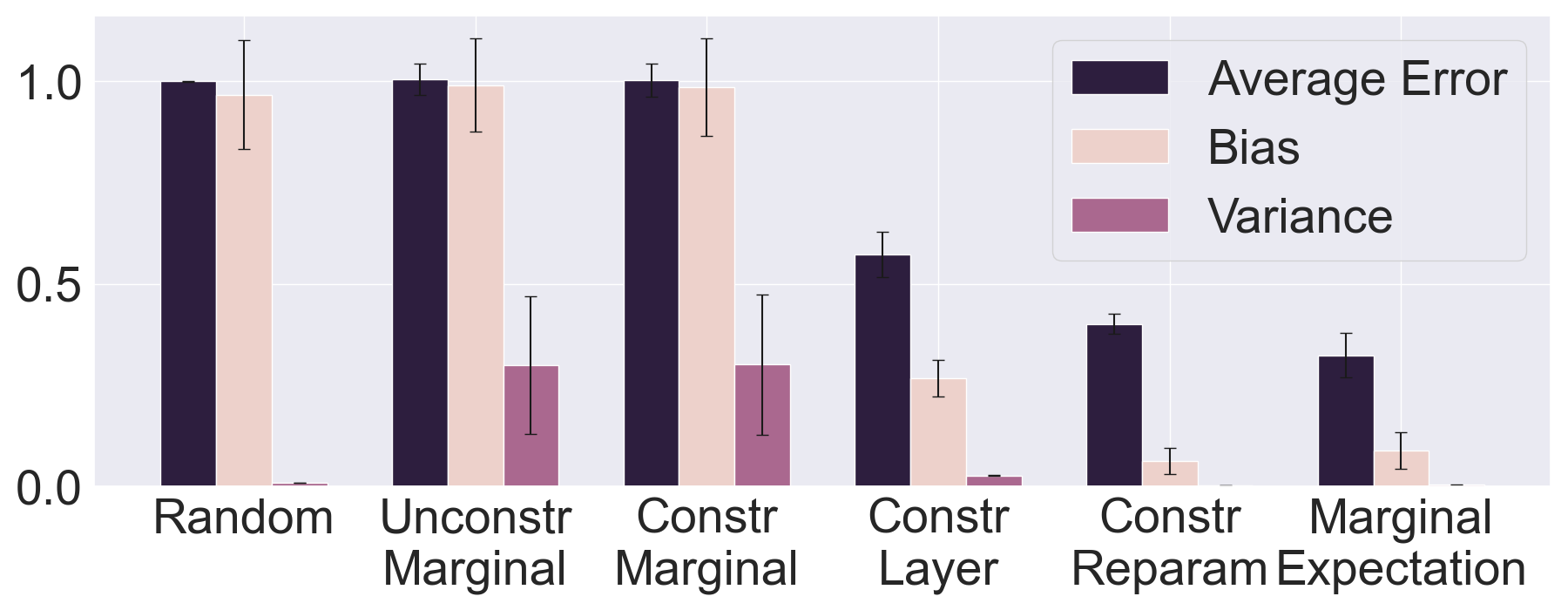}
    \end{subfigure}
    \hspace{0.08\textwidth}
    \begin{subfigure}
        \centering
        \includegraphics[height=0.125\textheight]{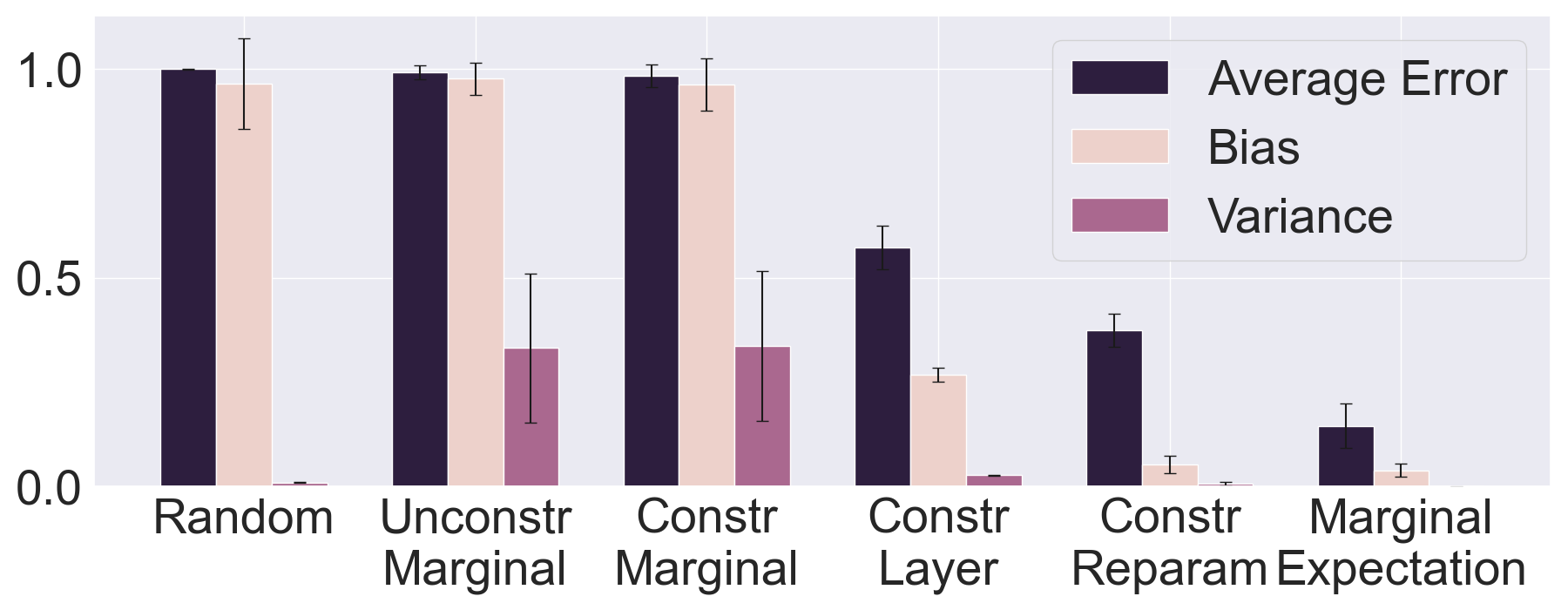}
    \end{subfigure}
    \caption{Comparisons of gradient estimators for point-wise loss $\ell$ being L1 loss (upper plot) and L2 loss (lower plot) applied to Gaussian variable are conducted. To compare the directions of the estimated and ground-truth gradients, we utilize the cosine distance. The bias, variance, and error of the gradient estimators are measured using a sample size of $10,000$.
    }
    \label{fig:Synthetic}
\end{figure}

\subsection{Comparison of Gradient Estimators}

We present a rigorous comparison of all aforementioned gradient estimator designs summarized in Table~\ref{tab:estimator summary}: 
we consider a synthetic setting where the ground truth gradients can be obtained by taking derivatives of a closed-form expected loss which will be described in Section~\ref{sec: Closed-form Expected Loss} such that we can compare how good the gradient estimations are for each estimator. 
The distance between the estimated and the ground truth gradient vectors is measured by cosine distance, defined as (1 -- cosine similarity). We evaluate the performance of gradient estimators on three metrics: bias, variance, and average error.

In addition to the two gradient estimators proposed in the previous sections, \textit{Constrained Marginal}, and \textit{Marginal Expectation}, 
we further propose a modified version of the reparameterization trick to expand the spectrum of estimator design. In this approach, unconstrained samples are initially generated using reparameterization trick, followed by a variance-weighted correction strategy to enforce constraints.
We further include three baseline estimators, \textit{Random}, \textit{Unconstrained Marginal}, \textit{Constrained Layer} as defined in Table~\ref{tab:estimator summary}. While the Constrained Layer introduced by ~\citet{stoian2024realistic} cannot be directly utilized as a gradient estimator, we integrate it with the reparameterization trick. Specifically, we utilize the reparameterization trick to facilitate backpropagation through the random sampling process and \textit{Constrained Layer} to enforce equality constraints. 

Results are shown in Figure~\ref{fig:Synthetic},
where \textit{Marginal Expectation} significantly outperforms the others in all cases.
\textit{Unconstrained Marginal} has similar performances to \textit{Random} which is expected since it discards the constraint information.
What is interesting is that \textit{Constrained Marginal}, even though it is informed by constraint,
it also performs as bad as \textit{Random} in terms of average error and bias.
In the Bernoulli setting, 
\textit{Marginal Expectation} and
\textit{Constrained Marginal} are the same estimator as shown in~\citet{ahmed2022simple}
while we show that
in the Gaussian setting,
the former is capable of providing decent gradient approximations while the latter is not. We refer the readers to Appendix~\ref{appendix:synthetic} for additional experimental details.

\section{Closed-Form Expected Loss}
\label{sec: Closed-form Expected Loss}

In this section, we turn to an opposite direction to explore when gradient estimators are not necessary.
It holds when the expected loss in Equation~\ref{eqn:Loss} admits closed-form expressions, allowing standard training to be applied and thus no gradient estimation is needed. 
For such cases to hold,
the first assumption we make is that the mapping $\decoder$ is an identity function, that is, $\yhat = \z$, as it can introduce high non-linearity. 
Then we show that when the element-wise loss $\ell$ is the L1 or L2 loss, the expected loss admits a closed-form expression as below.

\begin{proposition}[Gaussian Closed-form Expected Loss]\label{prop:closed form gaussian} 
    Let $\z \sim \mathcal{N} \left( \bm{\mu}, \bm{\Sigma} \right)$. 
    Let $\bm{y} = \left( y_1, \ldots, y_n \right)^T$ be the ground truth vector subject to the equality constraint $\bm{A} \bm{z} = \bm{k}$. 
    Then it holds that
    \begin{enumerate}[label=\roman*),noitemsep,topsep=0pt]
        \item when $\ell$ is L1 loss,
        $L(\btheta)$ has closed form 
        
        $
        \sum_{i=1}^n 
\overline{\bm{\Sigma}}_{i,i} \sqrt{\frac{2}{\pi}} 
e^{-\frac{(\overline{\bm{\mu}}_i - y_i)^2}{2 \overline{\bm{\Sigma}}_{i,i}^2}}
+ 
(\overline{\bm{\mu}}_i - y_i) \,\textit{erf} \left( \frac{\overline{\bm{\mu}}_i - y_i}{\sqrt{2} \, \overline{\bm{\Sigma}}_{i,i}} \right)
    $;
        \item when $\ell$ is L2 loss,
        $
        \sum_{i=1}^n
            \overline{\bm{\mu}}_i^2 + \overline{\bm{\Sigma}}_{i,i}^2
            -
            2 y_i \overline{\bm{\mu}}_i
            +
            y_i^2$,
    \end{enumerate}
    where $\overline{\bm{\mu}}$ and $\overline{\bm{\Sigma}}$ are defined above.
\end{proposition}
Later we will empirically show that when it is possible to derive the closed-form expected loss, it can leads to state-of-the-art generative performance.

\section{Experiments}
We conduct a comprehensive empirical evaluation 
to explore to what extent our proposed method leads to improved generative performance while providing guarantees on constraint satisfaction on both image generation benchmarks and scientific applications.

\subsection{VAE with Constrained Latent Space}
\label{sec:VAE_with_Constrained_Latent_Space}

To demonstrate the flexibility of our proposed gradient estimator, we consider an experiment setup where the VAE model has its latent space constrained by linear equality as regularization.
The VAE is trained on the MNIST dataset using the evidence lower bound~(ELBO) as objective, which consists of a reconstruction loss~(RL) and the KL divergence between a constrained approximate posterior $p_{\btheta} \left( \z \mid \bm{A} \z = \bm{k}, \x \right)$ and a prior of the latent space. 
The generative performance is evaluated using test negative likelihood, estimated using importance sampling ~\citep{burda2016importance}, negative ELBO, and reconstruction loss.

Experiment results are presented in Figure ~\ref{fig:vae_constrained_latent_space} where the estimator
\textit{Marginal Expectation}
outperforms the other estimators in all three metrics, consistent with synthetic experimental results in Figure~\ref{fig:Synthetic}.
\textit{Unconstrained Marginal} and \textit{Constrained Marginal} have similar performance, both better than \textit{Random}. \textit{Constrained Reparametrization} and \textit{Constrained Layer} exhibit similar performance, but both trailing behind \textit{Marginal Expectation} by a noticeable margin.
In the following experiments, we adopt \textit{Marginal Expectation} as the default gradient estimator for our approach.

\begin{figure}[t]
    \centering
    \includegraphics[width=0.45\textwidth]{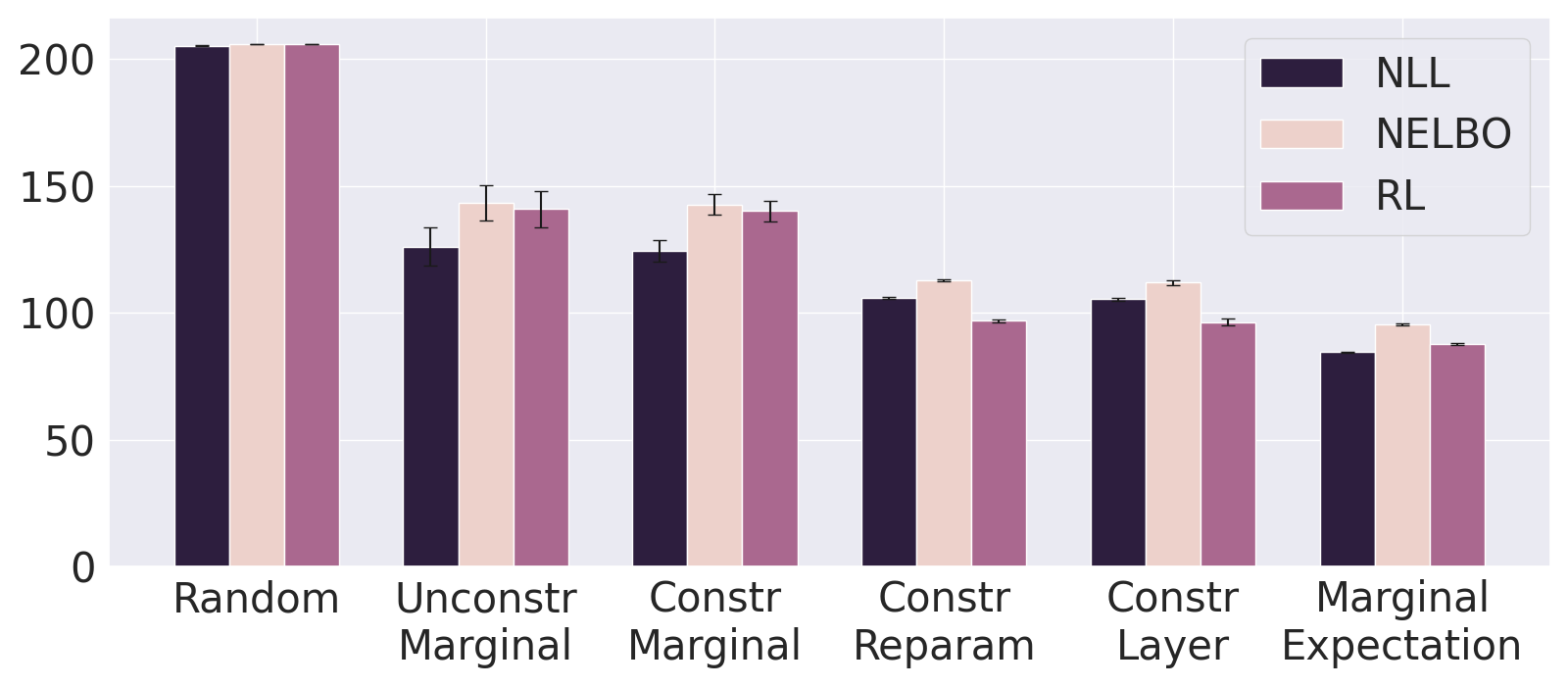}
    \caption{Comparison of gradient estimators for VAE with constrained latent space. Negative log-likelihood (NLL), negative ELBO (NELBO), and reconstruction loss (RL) are averaged over $5$ trials.}
    \label{fig:vae_constrained_latent_space}
\end{figure}

\subsection{Constrained Generation using VAE}
We consider a setting where the underlying data distribution is constrained by linear equality as domain knowledge.

\textbf{Setup.} We modify MNIST dataset by standardizing overall brightness of each image using a linear equality constraint. Three VAE models are considered: Vanilla VAE \citep{kingma2022autoencoding}, Ladder VAE \citep{sønderby2016ladder}, and Graph VAE \citep{he2018variational}. We compare the performance of these models and their constrained counterparts integrated with linear equality using estimator \textit{Marginal Expectation} as it shows the best performance. We also compare the integration of \textit{Constrained Layer} with these VAE models.  We refer the readers to Appendix~\ref{appendix:vae_constrained_data_generation} for model implementations and data modification. Similar to Section~\ref{sec:VAE_with_Constrained_Latent_Space}, the performance is evaluated from test log-likelihood~(LL), ELBO, and Reconstruction Loss~(RL). We also measure constraint violation rate, which calculates the proportion of reconstructed samples that violate constraints.

\textbf{Reults.} We find out that the unconstrained VAEs have high constraint violation rates. On the contrary, our method can constrain the model such that their generated data satisfy the constraint while also achieving better generative performance due to the inductive bias. Although \textit{Constrained Layer} can precisely enforce the constraint, it significantly diminishes the generative capability. Additionally, we also show that the addition of \textit{Marginal Expectation} has basically no impact on speed. We present the results in Appendix~\ref{appendix:vae_constrained_data_generation}.

\begin{table}[t]
    \centering
    \captionsetup{skip=6pt}
    \caption{Comparison on VAE generative performance. The constrained VAE models achieve similar or better generative ability while strictly satisfying the constraints, whereas the unconstrained counterparts have a high constraint violation rate.}
    \scalebox{0.68}{ 
    \begin{tabular}{l@{\hskip 4pt}r@{}l@{}r@{\hskip 8pt}r@{}l@{}r@{\hskip 8pt}r@{}l@{}r@{\hskip 8pt}r@{\hskip 1pt}l@{}r@{}l}
    \toprule
    \textbf{\textsc{Model}} & \multicolumn{3}{c}{\textbf{\textsc{LL}} $\uparrow$} & \multicolumn{3}{c}{\textbf{\textsc{ELBO}} $\uparrow$} & \multicolumn{3}{c}{\textbf{\textsc{RL}} $\downarrow$} & \multicolumn{3}{c}{\textbf{\textsc{Violation}} $\downarrow$} \\
    \midrule
    VAE & \text{-22.42} & \text{\hspace{2pt}$\pm$\hspace{2pt}} & \text{0.29} & \text{-23.41} & \text{\hspace{2pt}$\pm$\hspace{2pt}} & \text{0.22} & \text{15.00} & \text{\hspace{2pt}$\pm$\hspace{2pt}} & \text{0.46} & \makebox[1cm][r]{\text{0.30}} & \text{\hspace{2pt}$\pm$\hspace{2pt}} & \makebox[1cm][l]{\text{0.06}} \\
    VAE + CL & \text{-34.45} & \text{\hspace{2pt}$\pm$\hspace{2pt}} & \text{2.64} & \text{-40.89} & \text{\hspace{2pt}$\pm$\hspace{2pt}} & \text{9.37} & \text{37.11} & \text{\hspace{2pt}$\pm$\hspace{2pt}} & \text{9.35} & \makebox[1cm][r]{\textbf{0.00}} & \text{\hspace{2pt}$\pm$\hspace{2pt}} & \makebox[1cm][l]{\textbf{0.00}} \\
    ours & \textbf{-21.48} & \text{\hspace{2pt}$\pm$\hspace{2pt}} & \textbf{0.18} & \textbf{-22.62} & \text{\hspace{2pt}$\pm$\hspace{2pt}} & \textbf{0.07} & \textbf{12.79} & \text{\hspace{2pt}$\pm$\hspace{2pt}} & \textbf{0.11} & \makebox[1cm][r]{\textbf{0.00}} & \text{\hspace{2pt}$\pm$\hspace{2pt}} & \makebox[1cm][l]{\textbf{0.00}} \\
    \midrule
    Ladder VAE & \text{-24.25} & \text{\hspace{2pt}$\pm$\hspace{2pt}} & \text{0.07} & \text{-30.84} & \text{\hspace{2pt}$\pm$\hspace{2pt}} & \text{0.51} & \textbf{23.06} & \text{\hspace{2pt}$\pm$\hspace{2pt}} & \textbf{0.54} & \makebox[1cm][r]{\text{0.38}} & \text{\hspace{2pt}$\pm$\hspace{2pt}} & \makebox[1cm][l]{\text{0.02}} \\ 
    Ladder VAE + CL & \text{-36.83} & \text{\hspace{2pt}$\pm$\hspace{2pt}} & \text{0.56} & \text{-39.59} & \text{\hspace{2pt}$\pm$\hspace{2pt}} & \text{0.56} & \text{37.46} & \text{\hspace{2pt}$\pm$\hspace{2pt}} & \text{0.55} & \makebox[1cm][r]{\textbf{0.00}} & \text{\hspace{2pt}$\pm$\hspace{2pt}} & \makebox[1cm][l]{\textbf{0.00}} \\ 
    ours & \textbf{-23.86} & \text{\hspace{2pt}$\pm$\hspace{2pt}} & \textbf{0.06} & \textbf{-30.78} & \text{\hspace{2pt}$\pm$\hspace{2pt}} & \textbf{0.08} & \textbf{23.40} & \text{\hspace{2pt}$\pm$\hspace{2pt}} & \textbf{0.16} & \makebox[1cm][r]{\textbf{0.00}} & \text{\hspace{2pt}$\pm$\hspace{2pt}} & \makebox[1cm][l]{\textbf{0.00}} \\
    \midrule
    Graph VAE & \text{-22.74} & \text{\hspace{2pt}$\pm$\hspace{2pt}} & \text{0.11} & \text{-23.54} & \text{\hspace{2pt}$\pm$\hspace{2pt}} & \text{0.18} & \text{15.45} & \text{\hspace{2pt}$\pm$\hspace{2pt}} & \text{0.41} & \makebox[1cm][r]{\text{0.29}} & \text{\hspace{2pt}$\pm$\hspace{2pt}} & \makebox[1cm][l]{\text{0.09}} \\
    Graph VAE + CL & \text{-33.27} & \text{\hspace{2pt}$\pm$\hspace{2pt}} & \text{3.60} & \text{-33.27} & \text{\hspace{2pt}$\pm$\hspace{2pt}} & \text{5.84} & \text{28.29} & \text{\hspace{2pt}$\pm$\hspace{2pt}} & \text{6.40} & \makebox[1cm][r]{\textbf{0.00}} & \text{\hspace{2pt}$\pm$\hspace{2pt}} & \makebox[1cm][l]{\textbf{0.00}} \\
    ours & \textbf{-21.61} & \text{\hspace{2pt}$\pm$\hspace{2pt}} & \textbf{0.20} & \textbf{-22.53} & \text{\hspace{2pt}$\pm$\hspace{2pt}} & \textbf{0.06} & \textbf{12.73} & \text{\hspace{2pt}$\pm$\hspace{2pt}} & \textbf{0.21} & \makebox[1cm][r]{\textbf{0.00}} & \text{\hspace{2pt}$\pm$\hspace{2pt}} & \makebox[1cm][l]{\textbf{0.00}} \\
    \bottomrule
    \end{tabular}
    }
    \label{tab:VAEGenerative}
\end{table}

\subsection{Constrained Generation using Diffusion Models}
In this section, we consider a different DGM class, diffusion models~\citep{ho2020denoising, song2021denoising, song2019generative, lu2022dpm}.
We show that our method can be integrated into backward diffusion process and not only improve sample quality but also ensure constraint satisfaction. 

\textbf{Setup.}
We modify CIFAR 10 \citep{krizhevsky2009learning}, CelebA \citep{liu2015deep}, LSUN Church, and LSUN Cat \citep{yu2015lsun} datasets by standardizing the overall brightness of each image using linear equality constraints. We refer the readers to Appendix~\ref{appendix:Additional_Experimental_Detail_for_Diffusion_Model_Constrained_Data_Generation} for detailed modification of datasets. We evaluate the performance of diffusion models using three metrics: Fréchet Inception Distance (FID) \citep{heusel2017gans}, Inception Score (IS) \citep{salimans2016improved}, and the violation rate, which quantifies the proportion of generated samples that fail to satisfy the imposed constraints.

\textbf{DDPM.}
The model follows the standard DDPM \citep{ho2020denoising}, where a U-Net serves as the denoiser.  During training, Gaussian noise is incrementally added, and the U-Net is trained to predict and remove this noise by minimizing a reweighted variational lower bound. During inference, an iterative denoising process is performed using the learned U-Net to progressively generate samples from Gaussian noise by reversing diffusion. We incorporate our exact sampling methods into the backward diffusion process. Instead of predicting $p_{\theta} (\bm{x}_0 \mid \bm{x}_1)$, we generate $p_{\theta} (\bm{x}_0 \mid \bm{x}_1, \bm{A} \bm{x}_{0} = \bm{k})$. The results are presented in Table~\ref{tab:ddpm}.

\begin{table}[t]
    \centering
    \captionsetup{skip=6pt}
    \caption{Comparison of constrained and unconstrained models across datasets. We report FID (Fréchet Inception Distance), IS (Inception Score), and Violation metrics.}
    \begin{adjustbox}{max width=\columnwidth}
    \begin{tabular}{l l r r c}
    \toprule
    \textbf{\textsc{Dataset}} & \textbf{\textsc{Model}} & \textbf{\textsc{FID $\downarrow$}} & \textbf{\textsc{IS $\uparrow$}} & \textbf{\textsc{Violation $\downarrow$}} \\
    \midrule
    \multirow{2}{*}{\textbf{CIFAR}}    & Ours & \textbf{3.811} & 9.223 $\pm$ 0.130 & \textbf{0} \\
                                       & DDPM & 4.173 & \textbf{9.278} $\pm$ \textbf{0.116} & 0.999 \\
    \midrule
    \multirow{2}{*}{\textbf{CelebA}}   & Ours & \textbf{10.193} & \textbf{2.360} $\pm$ \textbf{0.016} & \textbf{0} \\
                                       & DDPM & 10.345 & 2.358 $\pm$ 0.030 & 0.999 \\
    \midrule
    \multirow{2}{*}{\textbf{LSUN Church}} & Ours & \textbf{4.779} & \textbf{2.471} $\pm$ \textbf{0.020} & \textbf{0} \\
                                          & DDPM & 4.945 & 2.460 $\pm$ 0.028 & 1.0 \\
    \midrule
    \multirow{2}{*}{\textbf{LSUN Cat}} & Ours & \textbf{12.489} & \textbf{4.711} $\pm$ \textbf{0.054} & \textbf{0} \\
                                       & DDPM & 12.913 & 4.705 $\pm$ 0.047 & 1.0 \\
    \bottomrule
    \end{tabular}
    \end{adjustbox}
    \label{tab:ddpm}
\end{table}

\textbf{DDIM.}
Directly integrating our exact sampling methods into the DDIM \citep{song2021denoising} last backward diffusion step ensures constraint satisfaction; however, it does not enhance generative performance. Inspired by recent success in incorporating guidance at intermediate backward diffusion steps \citep{yuan2023physdiff, liu2024image}, we also incorporate our method in selected backward diffusion steps under DDIM sampling mechanism. Using the CIFAR 10 dataset, we explore what would be the optimal schedule policy and the optimal number of constrained sampling steps. We refer the readers to Appendix~\ref{appendix:Additional_Experimental_Detail_for_Diffusion_Model_Constrained_Data_Generation} for additional details. We use this scheduling policy on all constrained models and compare with unconstrained counterparts in Table~\ref{tab:ddim}.

\textbf{Results.}
Our experimental results demonstrate that standard diffusion models rarely satisfy the imposed constraints, despite being trained on datasets where the data distribution is governed by these constraints. In contrast, diffusion models incorporating our method have guaranteed constraint satisfaction. Moreover, they exhibit superior generative performance, as evidenced by improved FID and IS metrics.

\begin{table}[t]
    \centering
    \captionsetup{skip=6pt}
    \caption{Comparison of constrained and unconstrained models across datasets using DDIM sampling.
    }
    \begin{adjustbox}{max width=\columnwidth}
    \begin{tabular}{l l r r c}
    \toprule
    \textbf{\textsc{Dataset}} & \textbf{\textsc{Model}} & \textbf{\textsc{FID $\downarrow$}} & \textbf{\textsc{IS $\uparrow$}} & \textbf{\textsc{Violation $\downarrow$}} \\
    \midrule
    \multirow{2}{*}{\textbf{CIFAR}}    & Ours & \textbf{7.972} & 8.585 $\pm$ 0.118 & \textbf{0} \\
                                       & DDIM & 8.123 & \textbf{8.646} $\pm$ \textbf{0.084} & 1.0 \\
    \midrule
    \multirow{2}{*}{\textbf{CelebA}}   & Ours & \textbf{12.389} & \textbf{2.367} $\pm$ \textbf{0.023} & \textbf{0} \\
                                       & DDIM & 12.417 & 2.366 $\pm$ 0.033 & 0.999 \\
    \midrule
    \multirow{2}{*}{\textbf{LSUN Church}} & Ours & \textbf{6.557} & \textbf{2.596} $\pm$ \textbf{0.033} & \textbf{0} \\
                                          & DDIM & 6.702 & 2.584 $\pm$ 0.027 & 0.9999 \\
    \midrule
    \multirow{2}{*}{\textbf{LSUN Cat}} & Ours & \textbf{18.902} & \textbf{4.857} $\pm$ \textbf{0.037} & \textbf{0} \\
                                       & DDIM & 18.958 & 4.849 $\pm$ 0.062 & 0.9999 \\
    \bottomrule
    \end{tabular}
    \end{adjustbox}
    \label{tab:ddim}
\end{table}

\subsection{Charge-Neutral Predictions}
\label{sec:Partial Charge Predictions for Metal-Organic Frameworks}

\begin{table}[t]
    \centering
    \captionsetup{skip=6pt}
    \caption{Performances of different methods for estimating partial charges on metal ions are presented. Compared to the baseline MPNN (variance), both the closed-form loss function and likelihood objective yield superior mean absolute deviation (MAD) results. 
    The same holds for their ensemble counterpart. 
    We find that
    ensemble methods (second block) notably boost the predictive performance in general.
    }
    \begin{adjustbox}{max width=\columnwidth}
    \begin{tabular}{l r @{\hskip 1pt}c@{\hskip 1pt} l r @{\hskip 1pt}c@{\hskip 1pt} l}
    \toprule
    \textbf{\textsc{Method}} & \multicolumn{3}{c}{\textbf{\textsc{MAD $\downarrow$}}} & \multicolumn{3}{c}{\textbf{\textsc{NLL $\downarrow$}}} \\
    neutrality enforcement & mean & \makebox[9pt][c]{$\!\!\pm\!\!$} & std & mean & \makebox[9pt][c]{$\!\!\pm\!\!$} & std \\
    \midrule
    Constrained Layer & 0.327 & \makebox[9pt][c]{$\!\!\pm\!\!$} & 0.004 & 103.522 & \makebox[9pt][c]{$\!\!\pm\!\!$} & 3.018 \\
    Constant Prediction & 0.324 & \makebox[9pt][c]{$\!\!\pm\!\!$} & 0.007 & \multicolumn{3}{c}{\multirow{1}{*}{---}} \\
    Element-mean (uniform) & 0.154 & \makebox[9pt][c]{$\!\!\pm\!\!$} & 0.002 & \multicolumn{3}{c}{\multirow{1}{*}{---}} \\
    Element-mean (variance) & 0.153 & \makebox[9pt][c]{$\!\!\pm\!\!$} & 0.002 & \multicolumn{3}{c}{\multirow{1}{*}{---}} \\
    MPNN (KKThPINN) & 0.0260 & \makebox[9pt][c]{$\!\!\pm\!\!$} & 0.0008 & 109.8 & \makebox[9pt][c]{$\!\!\pm\!\!$} & 6.9 \\
    MPNN (variance) & 0.0251 & \makebox[9pt][c]{$\!\!\pm\!\!$} & 0.0010 & -19.9 & \makebox[9pt][c]{$\!\!\pm\!\!$} & 71.1 \\
    Closed-form (ours) & \textbf{0.0245} & \makebox[9pt][c]{$\boldsymbol{\!\!\pm\!\!}$} & \textbf{0.0009} & & \makebox[9pt][c]{$>$} & 1{e+}7 \\
    Likelihood (ours) & 0.0248 & \makebox[9pt][c]{$\!\!\pm\!\!$} & 0.0008 & \textbf{-252} & \makebox[9pt][c]{$\boldsymbol{\!\!\pm\!\!}$} & \textbf{24.7} \\
    \midrule
    Constrained Layer (ens) & 0.319 & \makebox[9pt][c]{$\!\!\pm\!\!$} & 0.002 & 99.236 & \makebox[9pt][c]{$\!\!\pm\!\!$} & 2.3 \\
    MPNN (ens, KKThPINN) & 0.0244 & \makebox[9pt][c]{$\!\!\pm\!\!$} & 0.0006 & 57.29 & \makebox[9pt][c]{$\!\!\pm\!\!$} & 12.8 \\
    MPNN (ens, variance) & 0.0238 & \makebox[9pt][c]{$\!\!\pm\!\!$} & 0.0007 & -45.2 & \makebox[9pt][c]{$\!\!\pm\!\!$} & 55.8 \\
    Closed-form (ens, ours) & \textbf{0.0230} & \makebox[9pt][c]{$\boldsymbol{\!\!\pm\!\!}$} & \textbf{0.0008} & & \makebox[9pt][c]{$>$} & 1{e+}7 \\
    Likelihood (ens, ours) & \textbf{0.0231} & \makebox[9pt][c]{$\boldsymbol{\!\!\pm\!\!}$} & \textbf{0.0007} & \textbf{-180} & \makebox[9pt][c]{$\boldsymbol{\!\!\pm\!\!}$} & \textbf{38.3} \\
    \bottomrule
    \end{tabular}
    \end{adjustbox}
    \label{tab:Performance}
\end{table}

Metal-organic frameworks (MOFs) represent a class of materials with a wide range of applications in chemistry and materials science. Predicting properties of MOFs, such as partial charges on metal ions, is essential for understanding their reactivity and performance in chemical processes. However, it is challenging due to the complex interactions between metal ions and ligands and 
the requirement that the predictions need to satisfy the charge neutral constraint, that is, an exactly-zero constraint.

We adopt the same setting as \citet{raza2020message} where the model architecture uses the Message Passing Neural Network (MPNN) framework and incorporates equality constraint for charges, ensuring strict adherence to the critical constraint. The crystal structure of each MOF is modeled as an undirected graph, \( G = (\mathcal{V}, \mathcal{E}, \mathbf{X}) \), where \( \mathcal{V} \) represents the set of \( n = |\mathcal{V}| \) nodes corresponding to atoms, \( \mathcal{E} \) denotes the set of edges representing bonds, and \( \mathbf{X} \in \mathbb{R}^{d \times n} \) is the matrix of node features. The adjacency matrix \( \mathbf{D} \in \mathbb{R}^{n \times n} \) encodes edges with \( D_{uv} = 1 \) if nodes \( u \) and \( v \) are connected, else \( D_{uv} = 0 \). Our aim is to develop a function \( r \) that, given the graph \( G \), predicts the charge distribution across the nodes, which follows a Gaussian distribution: $(\mathbf{X}, \mathbf{D}) \mapsto r(\mathbf{X}, \mathbf{D}) = \mathbf{q}$. This function must adhere to the charge neutrality condition $\sum_{v=1}^{n} q_v = 0$, where \( q_v \) represents the charge associated with node \( v \) of the charge vector \( \mathbf{q} \in \mathbb{R}^n \). The model is trained using the element L1 loss.

The core innovation involves replacing the conventional L1 loss with the closed-form Gaussian loss as well as the negative log likelihood of the constrained multivariate Gaussian. The closed-form Gaussian loss penalizes deviations from the equality constraint while considering the probabilistic nature of Gaussian variables, and the negative log likelihood loss models the observed data with higher probability. Additionally, we also devise an ensemble methodology to enhance the predictive performance and robustness of our linear-equality constrained MPNN model. We apply the averaging aggregation technique to combine the predictions from two instances trained with variations in initialization. 

The prediction performance of our four proposed approaches is presented in Table~\ref{tab:Performance}. Results show that training using negative log likelihood loss and closed-form expected loss achieves better performance than MPNN (variance) which is considered to be the strongest baseline approach. When further combined with the ensemble method, our approach achieves significantly better predictions. We also combine \textit{Constrained Layer} and \textit{KKThPINN} with L1 loss functions. While \textit{Constrained Layer} exactly enforces the constraints, it significantly impairs predictive performance.

\begin{table}[t]
    \centering
    \captionsetup{skip=6pt}
    \caption{Comparison of models across \textbf{CSTR}, \textbf{plant}, and \textbf{distillation} tasks. The mean and standard deviation of MSE scaled by $10^{-4}$ are reported. All experiments are averaged for 10 times.}
    \begin{adjustbox}{max width=\columnwidth}
    \begin{tabular}{l r @{\hskip 1pt}c@{\hskip 1pt} l r @{\hskip 1pt}c@{\hskip 1pt} l r @{\hskip 1pt}c@{\hskip 1pt} l}
    \toprule
    \textbf{\textsc{Model}} & \multicolumn{3}{c}{\textbf{\textsc{CSTR}}} & \multicolumn{3}{c}{\textbf{\textsc{plant}}} & \multicolumn{3}{c}{\textbf{\textsc{distillation}}} \\
    \midrule
    ECNN & 20.6 & \makebox[9pt][c]{$\!\!\pm\!\!$} & 27.0 & 0.31 & \makebox[9pt][c]{$\!\!\pm\!\!$} & 0.23 & 1.94 & \makebox[9pt][c]{$\!\!\pm\!\!$} & 0.70 \\
    KKThPINN & 11.7 & \makebox[9pt][c]{$\!\!\pm\!\!$} & 20.3 &  0.11 & \makebox[9pt][c]{$\!\!\pm\!\!$} & 0.04 & 2.02 & \makebox[9pt][c]{$\!\!\pm\!\!$} & 0.94 \\
    NN & 18.3 & \makebox[9pt][c]{$\!\!\pm\!\!$} & 20.8 & 0.34 & \makebox[9pt][c]{$\!\!\pm\!\!$} & 0.64 & 1.99 & \makebox[9pt][c]{$\!\!\pm\!\!$} & 0.67 \\
    PINN & 260.8 & \makebox[9pt][c]{$\!\!\pm\!\!$} & 20.4 & 3.62 & \makebox[9pt][c]{$\!\!\pm\!\!$} & 1.94 & 40.9 & \makebox[9pt][c]{$\!\!\pm\!\!$} & 10.7 \\
    CL & 9.28 & \makebox[9pt][c]{$\!\!\pm\!\!$} & 3.56 & 0.58 & \makebox[9pt][c]{$\!\!\pm\!\!$} & 0.64 & 2.26 & \makebox[9pt][c]{$\!\!\pm\!\!$} & 1.19 \\
    Ours & \textbf{4.31} & \makebox[9pt][c]{$\!\!\pm\!\!$} & \textbf{1.58} & \textbf{0.09} & \makebox[9pt][c]{$\!\!\pm\!\!$} & \textbf{0.05} & \textbf{1.73} & \makebox[9pt][c]{$\!\!\pm\!\!$} & \textbf{0.70} \\
    \bottomrule
    \end{tabular}
    \end{adjustbox}
    \label{tab:chemical_engineering_result}
\end{table}

\subsection{Chemical Process Units and Subsystems}
Linear equality constraints are essential in chemical engineering, governing processes through principles like mass balance and stoichiometry \citep{chen2024hard}. High-fidelity simulations of the chemical systems could be computationally expensive due to the large number of differential and algebraic equations needed to solve. Thus, machine learning surrogate modeling has been a promising solution to provide physically accurate representations of these systems. We follow the experiment setting from \citet{chen2024hard} and conduct experiments on aspen models of a continuous stirred-tank reactor (\textbf{CSTR}) unit, an extractive distillation subsystem (\textbf{distillation}), and a chemical plant (\textbf{plant}).

While \citet{chen2024hard} imposes no distributional assumptions on the output space, our approach enforces a Gaussian distribution assumption, allowing the model to predict constrained mean and variance. Leveraging our theoretical framework, we sample exactly from the constrained distribution and train the model using closed-form expected loss functions. We compare the performance of our approach to several baselines ECNN \citep{chen2024hard}, \textit{KKThPINN} \citep{chen2024hard}, standard Feed Forward Neural Networks (NN), Physics Informed Neural Networks \citep{raissi2019physics}, and \textit{Constrained Layer} \citep{stoian2024realistic}. We refer the readers to the Appendix~\ref{appendix:chemical_process_units_and_subsystems} and \citet{chen2024hard} for detailed information regarding baseline implementations and experiment settings. As shown in Table~\ref{tab:chemical_engineering_result}, our model consistently outperforms the baselines by a significant margin. 
Furthermore, as detailed in Appendix~\ref{appendix:chemical_process_units_and_subsystems}, our method not only improves predictive accuracy but also achieves significantly faster convergence.

\subsection{Stock Investment}
We study a popular topic in financial engineering which leverages quantitative modeling, stochastic optimization, and predictive analytics to make data-driven decisions under uncertainty. Stock investment allocation \citep{zhang2020deep, butler2021integrating} involves determining the optimal allocation of investments across stocks based on predictions of future market trends. The objective is to create an allocation plan that maximizes returns, which requires that the sum of weights assigned to all stocks must equal 1. We conduct our experiments using historical data from the S\&P 500 index over a calendar year, the goal is to construct a portfolio that maximizes the Sharpe ratio \citep{sharpe1966mutual} for the next 120 trading days. We utilize a state-of-the-art time series prediction model from \citet{cao2020spectral}, which enforces the sum-to-one constraint using a default softmax activation. We assume the weights to follow Gaussian distributions by allowing shorting stocks. The models optimize stock investment allocation plans with their performance evaluated using the Sharpe ratio. $\text{Sharpe Ratio} = \frac{R_p - R_f}{\sigma_p}$, where $R_p$ denotes the expected return of the investment, $R_f$ is the risk-free rate of return, and $\sigma_p$ is the standard deviation of the investment's excess return. The results demonstrate that our model consistently outperforms the alternatives.

\begin{table}[t]
    \centering
    \captionsetup{skip=6pt}
    \caption{Comparison of models based on Sharpe ratio. We report the mean and standard deviation averaged across 10 runs.}
    \begin{adjustbox}{max width=\columnwidth}
    \begin{tabular}{l r @{\hskip 1pt}c@{\hskip 1pt} l}
    \toprule
    \textbf{\textsc{Model}} & \multicolumn{3}{c}{\textbf{\textsc{Sharpe ratio $\uparrow$}}} \\
    \midrule
    StemGNN & 1.5576 & \makebox[9pt][c]{$\!\!\pm\!\!$} & 0.3405 \\
    StemGNN-KKThPINN & 1.8092 & \makebox[9pt][c]{$\!\!\pm\!\!$} & 0.7055 \\
    StemGNN-CL & 1.5018 & \makebox[9pt][c]{$\!\!\pm\!\!$} & 0.3318 \\
    Ours & \textbf{1.9041} & \makebox[9pt][c]{$\!\!\pm\!\!$} & \textbf{0.2329} \\
    \bottomrule
    \end{tabular}
    \end{adjustbox}
    \label{tab:sharpe_ratio_comparison}
\end{table}

\section{Conclusion}
We introduced a principled framework for incorporating hard linear equality constraints into DGMs, addressing a fundamental challenge in generative modeling--ensuring constraint satisfaction while maintaining high data fidelity. 
Unlike existing methods that adjust individual samples post hoc, our approach directly constrains the distribution and enables end-to-end training of the constrained models through novel gradient estimators,
enabling flexible integration of constraints into various generative architectures. 
We further perform extensive empirical evaluations across diverse datasets and scientific applications. 
Our method outperforms baseline methods across multiple model classes, including VAEs, diffusion models, and graph neural networks, showcasing its flexibility and effectiveness.

\clearpage
\section*{Acknowledgements}
This work was funded in part by the DARPA ANSR program under award FA8750-23-2-0004, the DARPA CODORD program under award HR00112590089, NSF grant \#IIS-1943641, and gifts from Adobe Research, Cisco Research, and Amazon.

\section*{Impact Statement}
This paper presents work whose goal is to advance the field of Machine Learning. There are many potential societal consequences of our work, none of which we feel must be specifically highlighted here.

\bibliographystyle{icml2025}
\bibliography{references}

\begin{thebibliography}{53}
\providecommand{\natexlab}[1]{#1}
\providecommand{\url}[1]{\texttt{#1}}
\expandafter\ifx\csname urlstyle\endcsname\relax
  \providecommand{\doi}[1]{doi: #1}\else
  \providecommand{\doi}{doi: \begingroup \urlstyle{rm}\Url}\fi

\bibitem[Ahmed et~al.(2022)Ahmed, Teso, Chang, Van~den Broeck, and Vergari]{ahmed2022semantic}
Ahmed, K., Teso, S., Chang, K.-W., Van~den Broeck, G., and Vergari, A.
\newblock Semantic probabilistic layers for neuro-symbolic learning.
\newblock \emph{Advances in Neural Information Processing Systems}, 35:\penalty0 29944--29959, 2022.

\bibitem[Ahmed et~al.(2023)Ahmed, Zeng, Niepert, and Van~den Broeck]{ahmed2022simple}
Ahmed, K., Zeng, Z., Niepert, M., and Van~den Broeck, G.
\newblock Simple: A gradient estimator for k-subset sampling.
\newblock In \emph{Proceedings of the International Conference on Learning Representations (ICLR)}, may 2023.

\bibitem[Amos \& Kolter(2017)Amos and Kolter]{amos2017optnet}
Amos, B. and Kolter, J.~Z.
\newblock {O}pt{N}et: Differentiable optimization as a layer in neural networks.
\newblock In \emph{Proceedings of the 34th International Conference on Machine Learning}, volume~70 of \emph{Proceedings of Machine Learning Research}, pp.\  136--145. PMLR, 2017.

\bibitem[Badreddine et~al.(2022)Badreddine, Garcez, Serafini, and Spranger]{badreddine2022logic}
Badreddine, S., Garcez, A.~d., Serafini, L., and Spranger, M.
\newblock Logic tensor networks.
\newblock \emph{Artificial Intelligence}, 303:\penalty0 103649, 2022.

\bibitem[Burda et~al.(2016)Burda, Grosse, and Salakhutdinov]{burda2016importance}
Burda, Y., Grosse, R., and Salakhutdinov, R.
\newblock Importance weighted autoencoders.
\newblock In \emph{International Conference on Learning Representations (ICLR)}, 2016.

\bibitem[Butler \& Kwon(2021)Butler and Kwon]{butler2021integrating}
Butler, A. and Kwon, R.
\newblock Integrating prediction in mean-variance portfolio optimization.
\newblock \emph{Available at SSRN 3788875}, 2021.

\bibitem[Cao et~al.(2020)Cao, Wang, Duan, Zhang, Zhu, Huang, Tong, Xu, Bai, Tong, et~al.]{cao2020spectral}
Cao, D., Wang, Y., Duan, J., Zhang, C., Zhu, X., Huang, C., Tong, Y., Xu, B., Bai, J., Tong, J., et~al.
\newblock Spectral temporal graph neural network for multivariate time-series forecasting.
\newblock \emph{Advances in Neural Information Processing Systems}, 33:\penalty0 17766--17778, 2020.

\bibitem[Chen et~al.(2024)Chen, Flores, and Li]{chen2024hard}
Chen, H., Flores, G. E.~C., and Li, C.
\newblock Physics-informed neural networks with hard linear equality constraints.
\newblock \emph{Computers Chemical Engineering}, 189:\penalty0 108764, 2024.
\newblock ISSN 0098-1354.
\newblock \doi{https://doi.org/10.1016/j.compchemeng.2024.108764}.

\bibitem[Chen et~al.(2018)Chen, Song, Wainwright, and Jordan]{chen2018learning}
Chen, J., Song, L., Wainwright, M., and Jordan, M.
\newblock Learning to explain: An information-theoretic perspective on model interpretation.
\newblock In \emph{International conference on machine learning}, pp.\  883--892. PMLR, 2018.

\bibitem[De~Raedt et~al.(2007)De~Raedt, Kimmig, and Toivonen]{de2007problog}
De~Raedt, L., Kimmig, A., and Toivonen, H.
\newblock Problog: A probabilistic prolog and its application in link discovery.
\newblock In \emph{IJCAI 2007, Proceedings of the 20th international joint conference on artificial intelligence}, pp.\  2462--2467. IJCAI-INT JOINT CONF ARTIF INTELL, 2007.

\bibitem[Di~Liello et~al.(2020)Di~Liello, Ardino, Gobbi, Morettin, Teso, and Passerini]{di2020efficient}
Di~Liello, L., Ardino, P., Gobbi, J., Morettin, P., Teso, S., and Passerini, A.
\newblock Efficient generation of structured objects with constrained adversarial networks.
\newblock \emph{Advances in neural information processing systems}, 33:\penalty0 14663--14674, 2020.

\bibitem[Diaconis \& Zabell(1991)Diaconis and Zabell]{diaconis1991closed}
Diaconis, P. and Zabell, S.
\newblock Closed form summation for classical distributions: variations on a theme of de moivre.
\newblock \emph{Statistical Science}, pp.\  284--302, 1991.

\bibitem[Diligenti et~al.(2012)Diligenti, Gori, Maggini, and Rigutini]{diligenti2012bridging}
Diligenti, M., Gori, M., Maggini, M., and Rigutini, L.
\newblock Bridging logic and kernel machines.
\newblock \emph{Machine learning}, 86:\penalty0 57--88, 2012.

\bibitem[Djolonga \& Krause(2017)Djolonga and Krause]{Djolonga2017}
Djolonga, J. and Krause, A.
\newblock Differentiable learning of submodular models.
\newblock In Guyon, I., Luxburg, U.~V., Bengio, S., Wallach, H., Fergus, R., Vishwanathan, S., and Garnett, R. (eds.), \emph{Advances in Neural Information Processing Systems}, volume~30. Curran Associates, Inc., 2017.

\bibitem[Donti et~al.(2017)Donti, Amos, and Kolter]{donti2017task}
Donti, P., Amos, B., and Kolter, J.~Z.
\newblock Task-based end-to-end model learning in stochastic optimization.
\newblock In \emph{Advances in Neural Information Processing Systems}, pp.\  5484--5494, 2017.

\bibitem[Fischer et~al.(2019)Fischer, Balunovic, Drachsler-Cohen, Gehr, Zhang, and Vechev]{fischer2019dl2}
Fischer, M., Balunovic, M., Drachsler-Cohen, D., Gehr, T., Zhang, C., and Vechev, M.
\newblock Dl2: training and querying neural networks with logic.
\newblock In \emph{International Conference on Machine Learning}, pp.\  1931--1941. PMLR, 2019.

\bibitem[Garcez \& Lamb(2023)Garcez and Lamb]{garcez2023neurosymbolic}
Garcez, A.~d. and Lamb, L.~C.
\newblock Neurosymbolic ai: The 3 rd wave.
\newblock \emph{Artificial Intelligence Review}, 56\penalty0 (11):\penalty0 12387--12406, 2023.

\bibitem[Giunchiglia \& Lukasiewicz(2021)Giunchiglia and Lukasiewicz]{giunchiglia2021multi}
Giunchiglia, E. and Lukasiewicz, T.
\newblock Multi-label classification neural networks with hard logical constraints.
\newblock \emph{Journal of Artificial Intelligence Research}, 72:\penalty0 759--818, 2021.

\bibitem[Grover et~al.(2019)Grover, Wang, Zweig, and Ermon]{grover2019stochastic}
Grover, A., Wang, E., Zweig, A., and Ermon, S.
\newblock Stochastic optimization of sorting networks via continuous relaxations.
\newblock In \emph{International Conference on Learning Representations}, 2019.

\bibitem[He et~al.(2018)He, Gong, Marino, Mori, and Lehrmann]{he2018variational}
He, J., Gong, Y., Marino, J., Mori, G., and Lehrmann, A.
\newblock Variational autoencoders with jointly optimized latent dependency structure.
\newblock In \emph{International conference on learning representations}, 2018.

\bibitem[Hendriks et~al.(2020)Hendriks, Jidling, Wills, and Sch{\"o}n]{Hendriks2020LinearlyCN}
Hendriks, J.~N., Jidling, C., Wills, A.~G., and Sch{\"o}n, T.~B.
\newblock Linearly constrained neural networks.
\newblock \emph{ArXiv}, abs/2002.01600, 2020.

\bibitem[Heusel et~al.(2017)Heusel, Ramsauer, Unterthiner, Nessler, and Hochreiter]{heusel2017gans}
Heusel, M., Ramsauer, H., Unterthiner, T., Nessler, B., and Hochreiter, S.
\newblock Gans trained by a two time-scale update rule converge to a local nash equilibrium.
\newblock In \emph{Advances in Neural Information Processing Systems (NeurIPS)}, 2017.

\bibitem[Ho et~al.(2020)Ho, Jain, and Abbeel]{ho2020denoising}
Ho, J., Jain, A., and Abbeel, P.
\newblock Denoising diffusion probabilistic models.
\newblock \emph{arXiv preprint arxiv:2006.11239}, 2020.

\bibitem[Jang et~al.(2017)Jang, Gu, and Poole]{jang2016categorical}
Jang, E., Gu, S., and Poole, B.
\newblock Categorical reparameterization with gumbel-softmax.
\newblock In \emph{International Conference on Learning Representations}, 2017.

\bibitem[Kim et~al.(2016)Kim, Sabharwal, and Ermon]{kim2016exact}
Kim, C., Sabharwal, A., and Ermon, S.
\newblock Exact sampling with integer linear programs and random perturbations.
\newblock In \emph{Proceedings of the AAAI Conference on Artificial Intelligence}, volume~30, 2016.

\bibitem[Kingma \& Welling(2013)Kingma and Welling]{kingma2022autoencoding}
Kingma, D.~P. and Welling, M.
\newblock Auto-encoding variational bayes.
\newblock \emph{CoRR}, abs/1312.6114, 2013.

\bibitem[Krizhevsky(2009)]{krizhevsky2009learning}
Krizhevsky, A.
\newblock Learning multiple layers of features from tiny images.
\newblock Technical Report TR-2009, University of Toronto, 2009.

\bibitem[Lamboni(2022)]{LAMBONI2022199}
Lamboni, M.
\newblock Efficient dependency models: Simulating dependent random variables.
\newblock \emph{Mathematics and Computers in Simulation}, 200:\penalty0 199--217, 2022.
\newblock ISSN 0378-4754.
\newblock \doi{https://doi.org/10.1016/j.matcom.2022.04.018}.

\bibitem[Liu et~al.(2024)Liu, Niepert, and den Broeck]{liu2024image}
Liu, A., Niepert, M., and den Broeck, G.~V.
\newblock Image inpainting via tractable steering of diffusion models.
\newblock In \emph{The Twelfth International Conference on Learning Representations}, 2024.

\bibitem[Liu et~al.(2015)Liu, Luo, Wang, and Tang]{liu2015deep}
Liu, Z., Luo, P., Wang, X., and Tang, X.
\newblock Deep learning face attributes in the wild.
\newblock In \emph{Proceedings of the IEEE International Conference on Computer Vision (ICCV)}, pp.\  3730--3738, 2015.

\bibitem[Lu et~al.(2022)Lu, Zhou, Bao, Chen, Li, and Zhu]{lu2022dpm}
Lu, C., Zhou, Y., Bao, F., Chen, J., Li, C., and Zhu, J.
\newblock Dpm-solver: A fast ode solver for diffusion probabilistic model sampling in around 10 steps.
\newblock \emph{arXiv preprint arXiv:2206.00927}, 2022.

\bibitem[Maddison et~al.(2017)Maddison, Mnih, and Teh]{maddison2016concrete}
Maddison, C.~J., Mnih, A., and Teh, Y.~W.
\newblock The concrete distribution: A continuous relaxation of discrete random variables.
\newblock In \emph{International Conference on Learning Representations}, 2017.

\bibitem[Misino et~al.(2022)Misino, Marra, and Sansone]{misino2022vael}
Misino, E., Marra, G., and Sansone, E.
\newblock Vael: Bridging variational autoencoders and probabilistic logic programming.
\newblock \emph{Advances in Neural Information Processing Systems}, 35:\penalty0 4667--4679, 2022.

\bibitem[Raissi et~al.(2019)Raissi, Perdikaris, and Karniadakis]{raissi2019physics}
Raissi, M., Perdikaris, P., and Karniadakis, G.~E.
\newblock Physics-informed neural networks: A deep learning framework for solving forward and inverse problems involving nonlinear partial differential equations.
\newblock \emph{Journal of Computational Physics}, 378:\penalty0 686--707, 2019.

\bibitem[Raza et~al.(2020)Raza, Sturluson, Simon, and Fern]{raza2020message}
Raza, A., Sturluson, A., Simon, C.~M., and Fern, X.
\newblock Message passing neural networks for partial charge assignment to metal--organic frameworks.
\newblock \emph{The Journal of Physical Chemistry C}, 124\penalty0 (35):\penalty0 19070--19082, 2020.

\bibitem[Salimans et~al.(2016)Salimans, Goodfellow, Zaremba, Cheung, Radford, and Chen]{salimans2016improved}
Salimans, T., Goodfellow, I., Zaremba, W., Cheung, V., Radford, A., and Chen, X.
\newblock Improved techniques for training gans.
\newblock In \emph{Advances in Neural Information Processing Systems (NeurIPS)}, 2016.

\bibitem[Sharpe(1966)]{sharpe1966mutual}
Sharpe, W.~F.
\newblock Mutual fund performance.
\newblock \emph{Journal of Business}, 39\penalty0 (1):\penalty0 119--138, 1966.

\bibitem[Shukla et~al.(2024)Shukla, Zeng, Ahmed, and Van~den Broeck]{shukla2024unified}
Shukla, V., Zeng, Z., Ahmed, K., and Van~den Broeck, G.
\newblock A unified approach to count-based weakly supervised learning.
\newblock \emph{Advances in Neural Information Processing Systems}, 36, 2024.

\bibitem[S\o~nderby et~al.(2016)S\o~nderby, Raiko, Maal\o~e, S\o~nderby, and Winther]{sønderby2016ladder}
S\o~nderby, C.~K., Raiko, T., Maal\o~e, L., S\o~nderby, S. r.~K., and Winther, O.
\newblock Ladder variational autoencoders.
\newblock In Lee, D., Sugiyama, M., Luxburg, U., Guyon, I., and Garnett, R. (eds.), \emph{Advances in Neural Information Processing Systems}, volume~29. Curran Associates, Inc., 2016.

\bibitem[Song et~al.(2021)Song, Meng, and Ermon]{song2021denoising}
Song, J., Meng, C., and Ermon, S.
\newblock Denoising diffusion implicit models.
\newblock In \emph{International Conference on Learning Representations}, 2021.

\bibitem[Song \& Ermon(2019)Song and Ermon]{song2019generative}
Song, Y. and Ermon, S.
\newblock Generative modeling by estimating gradients of the data distribution.
\newblock In \emph{Advances in Neural Information Processing Systems}, pp.\  11895--11907, 2019.

\bibitem[Stoian et~al.(2023)Stoian, Giunchiglia, and Lukasiewicz]{stoian2024exploiting}
Stoian, M.~C., Giunchiglia, E., and Lukasiewicz, T.
\newblock Exploiting t-norms for deep learning in autonomous driving.
\newblock In d'Avila Garcez, A.~S., Besold, T.~R., Gori, M., and Jiménez-Ruiz, E. (eds.), \emph{Proceedings of the 17th International Workshop on Neural-Symbolic Learning and Reasoning, NeSy 2023, La Certosa di Pontignano, Siena, Italy, 3--5 July 2023}, pp.\  369--380, July 2023.

\bibitem[Stoian et~al.(2024)Stoian, Dyrmishi, Cordy, Lukasiewicz, and Giunchiglia]{stoian2024realistic}
Stoian, M.~C., Dyrmishi, S., Cordy, M., Lukasiewicz, T., and Giunchiglia, E.
\newblock How realistic is your synthetic data? constraining deep generative models for tabular data.
\newblock In \emph{The Twelfth International Conference on Learning Representations}, 2024.

\bibitem[Tschiatschek et~al.(2018)Tschiatschek, Sahin, and Krause]{Tschiatschek2018DifferentiableSM}
Tschiatschek, S., Sahin, A., and Krause, A.
\newblock Differentiable submodular maximization.
\newblock In \emph{International Joint Conference on Artificial Intelligence}, 2018.

\bibitem[Vrins(2018)]{risks6030064}
Vrins, F.
\newblock Sampling the multivariate standard normal distribution under a weighted sum constraint.
\newblock \emph{Risks}, 6\penalty0 (3), 2018.
\newblock ISSN 2227-9091.
\newblock \doi{10.3390/risks6030064}.

\bibitem[Wang et~al.(2023)Wang, Zhang, Guo, Chen, Yang, and Yan]{WangICML23}
Wang, R., Zhang, Y., Guo, Z., Chen, T., Yang, X., and Yan, J.
\newblock {LinSATNet}: The positive linear satisfiability neural networks.
\newblock In \emph{International Conference on Machine Learning (ICML)}, 2023.

\bibitem[Wilder(2019)]{wilder2019melding}
Wilder, B.
\newblock Melding the data-decisions pipeline: Decision-focused learning for combinatorial optimization.
\newblock In \emph{Proceedings of the 33rd AAAI Conference on Artificial Intelligence}, 2019.

\bibitem[Xie \& Ermon(2019)Xie and Ermon]{Sang2019reparameterizable}
Xie, S.~M. and Ermon, S.
\newblock Reparameterizable subset sampling via continuous relaxations.
\newblock \emph{International Joint Conference on Artificial Intelligence (IJCAI)}, 2019.

\bibitem[Xu et~al.(2018)Xu, Zhang, Friedman, Liang, and Broeck]{xu2018semantic}
Xu, J., Zhang, Z., Friedman, T., Liang, Y., and Broeck, G.
\newblock A semantic loss function for deep learning with symbolic knowledge.
\newblock In \emph{International conference on machine learning}, pp.\  5502--5511. PMLR, 2018.

\bibitem[Yu et~al.(2015)Yu, Seff, Zhang, Song, Funkhouser, and Xiao]{yu2015lsun}
Yu, F., Seff, A., Zhang, Y., Song, S., Funkhouser, T., and Xiao, J.
\newblock Lsun: Construction of a large-scale image dataset using deep learning with humans in the loop.
\newblock In \emph{Proceedings of the IEEE Conference on Computer Vision and Pattern Recognition (CVPR) Workshops}, pp.\  1--10, 2015.

\bibitem[Yuan et~al.(2023)Yuan, Song, Iqbal, Vahdat, and Kautz]{yuan2023physdiff}
Yuan, Y., Song, J., Iqbal, U., Vahdat, A., and Kautz, J.
\newblock Physdiff: Physics-guided human motion diffusion model.
\newblock In \emph{Proceedings of the IEEE/CVF International Conference on Computer Vision (ICCV)}, 2023.

\bibitem[Zhang et~al.(2023)Zhang, Li, Meng, Chang, and Van~den Broeck]{zhang2023paradox}
Zhang, H., Li, L.~H., Meng, T., Chang, K.-W., and Van~den Broeck, G.
\newblock On the paradox of learning to reason from data.
\newblock In \emph{Proceedings of the Thirty-Second International Joint Conference on Artificial Intelligence}, pp.\  3365--3373, 2023.

\bibitem[Zhang et~al.(2020)Zhang, Zohren, and Roberts]{zhang2020deep}
Zhang, Z., Zohren, S., and Roberts, S.
\newblock Deep learning for portfolio optimization.
\newblock \emph{The Journal of Financial Data Science}, 2\penalty0 (4):\penalty0 8--20, 2020.

\end{thebibliography}
\clearpage

\appendix

\section{Additional Experiment Details in Synthetic Settings}
\label{appendix:synthetic}
We carried out a series of experiments to analyze the effectiveness of our gradient estimator from Gaussian variable. Our focus lies on three pivotal metrics: bias, variance, and the average error. Since, we only care about the direction of the gradients, we employed the cosine distance, namely 1 -- cosine similarity, to measure the deviation of our gradient estimators from the ground truth vector. The ground truth are sampled from $\mathcal{N} (\bm{0}, \bm{I} )$ satisfying the constraint. We randomly generated $20$ sets of parameters and calculated the metrics for each set. Then, we take average of these $20$ repeats and computed their standard deviations. The randomly generated gradients are sampled from $\mathcal{N} (\bm{0}, \bm{I} )$.

\begin{itemize}
    \item \textbf{bias}: $1 - cos \left( \frac{\sum_j^n h_j}{n}, h_{gt}\right)$
    \item \textbf{variance}: $var\left( \left\{ 1 - cos \left( h_i, \frac{\sum_j^n h_j}{n} \right) \right\}_{i=1}^n \right)$
    \item \textbf{averaged error}: $\frac{\sum_{i=1}^n 1 - cos \left( h_i, h_{gt} \right)}{n}$ 
\end{itemize}
where $h_i$ denotes the approximated gradient and $h_{gt}$ denotes the ground truth gradient.

\section{Additional Experimental Details for VAE Constrained Latent Space}
\label{appendix:vae_constrained_latent_space}
\paragraph{Model} We present the model architecture used in the experiment. \newline
\textbf{Encoders:}
\[
\begin{aligned}
&\text{fc}(\text{input\_size}, 512) \rightarrow  \text{ReLU} \rightarrow \text{fc}(512, 256) \rightarrow  \text{ReLU} \\
&\rightarrow \text{fc}(512, z_{dim}) 
\end{aligned}
\]
\textbf{Sampling:}
We use two separate $\text{fc}(z_{dim}, z_{dim})$ for predicting the mean and log variance of the latent distribution. We sample exactly from the constrained distribution. \newline
\textbf{Decoders:}
\[
\begin{aligned}
&\text{fc}(z_{dim}, 256) \rightarrow  \text{ReLU} \rightarrow \text{fc}(256, 512) \rightarrow  \text{ReLU} \\
&\rightarrow \text{fc}(512, \text{input\_size}) \rightarrow \texttt{output\_function}
\end{aligned}
\]
\texttt{output\_function} is \texttt{sigmoid()} predicting the mean of Bernoulli Observations.

\paragraph{Training}
All models were implemented with PyTorch and trained using the Adam optimizer with a mini-batch size of 128 and learning rate 0.0001. All models are trained with $100$ epochs.

\section{Additional Experimental Details for VAE Constrained Data Generation}
\label{appendix:vae_constrained_data_generation}

\begin{figure}[t]
    \centering
    \includegraphics[width=\columnwidth]{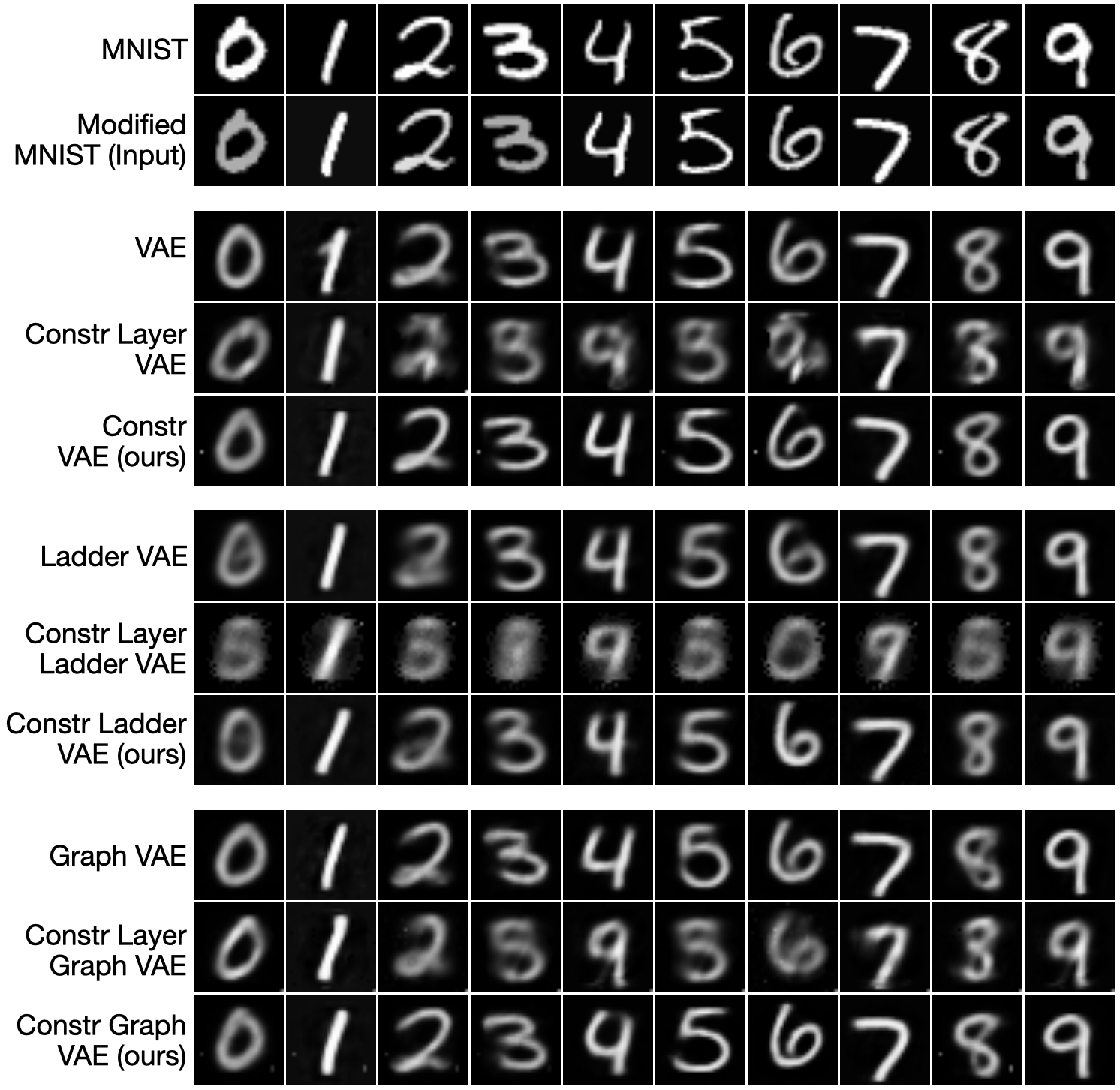}
    \caption{
    The first block displays the original MNIST images and the ones modified by the brightness constraint as inputs.
    For the following blocks, each displays the reconstructed images by different VAE architectures.
    Within each block,
    the first row is generated by the unconstrained VAE,
    the second by VAE constrained by the baseline Constrained Layer
    and the last one by VAE constrained by our method.
    }
    \label{fig:mnist example}
\end{figure}

\begin{table}[t]
    \centering
    \captionsetup{skip=6pt}
    \caption{Comparison on VAE Constraint Violation and Training Time. The tabel reports the average training time for one epoch. We observe that our approach to enforce the constraints does not cause significant increase in training time. For Vanilla VAE and Ladder VAE, the increase in training time is less than 1 seconds. }
    \begin{tabular}{c c}
    \toprule
    \textbf{\textsc{Algorithm}}  & \textbf{\textsc{Training Time}} $\downarrow$ \\
    \midrule
    VAE  & $\mathbf{3.94} \boldsymbol{\pm} \mathbf{0.14}$  \\
    Constrained VAE  & $4.21 \pm 0.11$ \\
    Ladder VAE  & $\mathbf{10.11} \boldsymbol{\pm} \mathbf{0.37}$ \\
    Constrained Ladder VAE  & $10.90 \pm 0.47$ \\
    Graph VAE  & $\mathbf{44.11} \boldsymbol{\pm} \mathbf{0.50}$ \\
    Constrained Graph VAE  & $46.61 \pm 0.59$ \\
    \bottomrule
    \end{tabular}
    
    \label{tab:VAETraining}

\end{table}

\paragraph{Model} We adopted a model architecture similar to ~\citep{he2018variational}. \newline
\textbf{Encoders:}
\[
\begin{aligned}
&\text{fc}(\text{input\_size}, 512) \rightarrow \text{batch\_norm} \rightarrow \text{ELU} \rightarrow \text{fc}(512, 512) \\
&\rightarrow \text{batch\_norm} \rightarrow \text{ELU} \rightarrow \text{fc}(512, 256) \rightarrow \text{batch\_norm} \rightarrow \\ 
&\text{ELU} \rightarrow \text{fc}(256, 128)
\end{aligned}
\]
\textbf{Decoders:}
\[
\begin{aligned}
&\text{fc}(N', 256) \rightarrow \text{batch\_norm} \rightarrow \text{ELU} \rightarrow \text{fc}(256, 512) \rightarrow \\
&\text{batch\_norm} \rightarrow \text{ELU} \rightarrow \text{fc} (512, 512) \rightarrow \text{batch\_norm} \rightarrow \\
&\text{ELU} \rightarrow \text{fc}(512, \text{input\_size}) \rightarrow \text{output\_function()}
\end{aligned}
\]
\texttt{output\_function} is \texttt{sigmoid()} predicting $\mu$, \texttt{fc}(\texttt{input\_size}, \texttt{input\_size}) predicting log var of Gaussian observations.

\paragraph{Training}
All models were implemented with PyTorch and trained using the Adam optimizer with a mini-batch size of 128. We conducted hyperparameter tuning on learning rate and adopt the best learning rate in the final model. All models are trained with $1000$ epochs.

\paragraph{Dataset}
We modify the original MNIST dataset by introducing linear equality constraint on every image. We constraint the sum of all pixel values in every image to be $100$, which is approximately the mean and median of all images in MNIST. We first scale the pixel values so the sum is $100$. To enforce the pixel values in the range $[0, 1]$, we uniformly distribute the extra pixel values to white pixels.

\paragraph{Training Time}
In order to show that \textit{Marginal Expectation} adds minimal training time, we measure the training time of 1 epoch and report the average training time.  We record the average training time of all the models for one epoch. All results are averaged over 5 independent runs. The results are summarized in Table~\ref{tab:VAEGenerative}. The results show that the constrained versions are less than $10\%$ slower than the unconstrained version. \textit{Marginal Expectation} adds less than 1 second of additional training time for VAE and Ladder VAE.

\section{Additional Experimental Detail for Diffusion Model Constrained Data Generation}
\label{appendix:Additional_Experimental_Detail_for_Diffusion_Model_Constrained_Data_Generation}

\subsection{Algorithm}
We present the algorithm for incorporating our exact sampling method in selected intermediate backward diffusion steps in Algorithm~\ref{algo:ddim_sample}. 


\begin{algorithm}
\caption{DDIM with Exact Sampling during Inference.}
\label{algo:ddim_sample}
\textbf{Require:} $\tau = \{\tau_0, \tau_1, ..., \tau_K\}$ (diffusion timestamp sequence, where $\tau_0 = 0, \cdots \tau_K = T$.
\begin{algorithmic}[1]
    \STATE $\bm{x}_{\tau_{K}} \sim \mathcal{N}(0, I)$
    \FOR{$i = K, K - 1, \dots, 1, 0$}
        \IF{Exact Sampling}
            \STATE $\bm{x}_0^{\tau_{i-1}} \sim \mathcal{N}(\bm{x}_0^{\tau_{i-1}}; \boldsymbol{\mu}_{\theta}(\bm{x}_{\tau_{i}}, t), \bm{\Sigma}_{\theta}, \bm{A} \boldsymbol{\mu}_{\theta} = \bm{k})$
        \ELSE
            \STATE $\bm{x}_0^{\tau_{i-1}} \sim \mathcal{N}(\bm{x}_0^{\tau_{i-1}}; \boldsymbol{\mu}_{\theta}(\bm{x}_{\tau_{i}}, t), \bm{\Sigma}_{\theta})$
        \ENDIF
        \STATE $\bm{x}_{\tau_{i-1}} \sim p(\bm{x}_{\tau_{i-1}} \mid \bm{x}_0^{\tau_{i-1}}, \bm{x}_{\tau_{i}})$
    \ENDFOR
    \STATE Return $\bm{x}_{\tau_0}$
\end{algorithmic}
\end{algorithm}

\subsection{Model Architecture}
The denoising model used in DDPM is a U-Net \citep{ho2020denoising} architecture. The details of the model used for CIFAR 10 are as follows:
\begin{itemize}
    \item \textbf{Channels ($ch$)}: 128
    \item \textbf{Channel Multipliers ($ch\_mult$)}: \{1, 2, 2, 2\}
    \item \textbf{Dropout}: 0.1
    \item \textbf{Number of Residual Blocks ($num\_res\_blocks$)}: 2
    \item \textbf{Number of Attention Blocks ($attn$)}: 2
\end{itemize}

The details of the model used for CELEBA, LSUN Church, and LSUN Cat are as follows:
\begin{itemize}
    \item \textbf{Channels ($ch$)}: 128
    \item \textbf{Channel Multipliers ($ch\_mult$)}: \{1, 1, 2, 2, 4\}
    \item \textbf{Dropout}: 0.1
    \item \textbf{Number of Residual Blocks ($num\_res\_blocks$)}: 2
    \item \textbf{Number of Attention Blocks ($attn$)}: 2
\end{itemize}

\subsection{Training and Evaluation}
We trained the DDPM using a distributed setup on NVIDIA V100 GPUs with 32GB of memory. For the CIFAR-10 dataset, we used 4 GPUs, while training on CELEBA-HQ, LSUN Church, and LSUN Cat utilized 6 GPUs. Training was conducted with exponential moving average (EMA) applied to the model parameters, using a decay factor of 0.9999. The number of diffusion steps was fixed at $T = 1000$ without a hyperparameter sweep, employing a linear schedule for the noise variance from $\beta_1 = 10^{-4}$ to $\beta_T = 0.02$. CIFAR-10 was trained for 800,000 steps, CelebA-HQ for 500,000 steps, LSUN Cat for 1.8 million steps, and LSUN Church for 1.2 million steps. Images from CelebA-HQ, LSUN Cat, and LSUN Church datasets were uniformly downsampled to $128 \times 128$ resolution.

For evaluation, Inception and Fréchet Inception Distance (FID) scores were calculated on 50,000 samples. 

\subsection{Data}
We modify the original datasets by introducing linear equality constraint on every channel for each image. We constraint the sum of all pixel values in every image to be the mean or median of all images in the dataset.

\subsection{Additional Experiment Results}
We first explore what would be the optimal schedule policy by comparing against: 1.) \textit{Uniform N}: Distributing N constrained exact sampling evenly across diffusion steps; 2.) \textit{Start M, End N} Placing M consecutive constrained exact sampling at the start and N at the end of the diffusion process; 3.) \textit{Start N, Space S}: Placing N constrained exact sampling at the start with a spacing of S; 4.) \textit{End N, Space S}: Placing N constrained exact sampling at the end with a spacing of S. The results are shown in Table~\ref{tab:schedule_comparison}, which suggest that placement at the end exactly enforces constraint satisfaction, while constraint layers placed at the start produce higher IS. For optimal balance, we adopt the \textit{Start N End N} schedule, which places $N$ correction steps at both the beginning and the end of the diffusion process. Next, we test the optimal number of $N$ under this schedule in Table~\ref{tab:num_constraint_layer}. 

Since $N=2$ achieves the lowest FID and lowest IS and $N=4$ achieves highest FID and highest IS, we adopt $N=3$ to balance these two metrics.

\begin{table}[t]
    \centering
    \captionsetup{skip=6pt}
    \caption{Comparison of schedules based on FID, IS, and Violation metrics. The experiments are conducted on CIFAR-10 dataset.}
    \begin{adjustbox}{max width=\columnwidth}
    \begin{tabular}{l r @{\hskip 1pt}c@{\hskip 1pt} l r @{\hskip 1pt}c@{\hskip 1pt} l r}
    \toprule
    \textbf{Schedule} & \multicolumn{3}{c}{\textbf{FID $\downarrow$}} & \multicolumn{3}{c}{\textbf{IS $\uparrow$}} & \textbf{Violation $\downarrow$} \\
    \midrule
    Uniform 4 & 7.979 & \makebox[9pt][c]{} &  & 8.589 & \makebox[9pt][c]{$\!\!\pm\!\!$} & 0.121 & 0 \\
    Start 3 End 1 & 8.049 & \makebox[9pt][c]{} &  & 8.612 & \makebox[9pt][c]{$\!\!\pm\!\!$} & 0.136 & 0 \\
    Start 2 End 2 & 7.823 & \makebox[9pt][c]{} &  & 8.570 & \makebox[9pt][c]{$\!\!\pm\!\!$} & 0.069 & 0 \\
    Start 1 End 3 & 7.912 & \makebox[9pt][c]{} &  & 8.616 & \makebox[9pt][c]{$\!\!\pm\!\!$} & 0.088 & 0 \\
    Start 4 Space 1 & 7.999 & \makebox[9pt][c]{} &  & 8.670 & \makebox[9pt][c]{$\!\!\pm\!\!$} & 0.108 & 0.9999 \\
    Start 4 Space 2 & 8.040 & \makebox[9pt][c]{} &  & 8.580 & \makebox[9pt][c]{$\!\!\pm\!\!$} & 0.102 & 0.9999 \\
    Start 4 Space 3 & 8.092 & \makebox[9pt][c]{} &  & 8.564 & \makebox[9pt][c]{$\!\!\pm\!\!$} & 0.080 & 0.9999 \\
    End 4 Space 1 & 8.061 & \makebox[9pt][c]{} &  & 8.580 & \makebox[9pt][c]{$\!\!\pm\!\!$} & 0.098 & 0 \\
    End 4 Space 2 & 7.924 & \makebox[9pt][c]{} &  & 8.603 & \makebox[9pt][c]{$\!\!\pm\!\!$} & 0.096 & 0 \\
    End 4 Space 3 & 8.016 & \makebox[9pt][c]{} &  & 8.594 & \makebox[9pt][c]{$\!\!\pm\!\!$} & 0.133 & 0 \\
    \bottomrule
    \end{tabular}
    \end{adjustbox}
    \label{tab:schedule_comparison}
\end{table}

\begin{table}[t]
    \centering
    \captionsetup{skip=6pt}
    \caption{Performance metrics across different $N$ under the Start N End N schedule. The experiments are conducted on CIFAR-10 dataset.}
    \begin{adjustbox}{max width=\columnwidth}
    \begin{tabular}{r r @{\hskip 1pt}c@{\hskip 1pt} l r @{\hskip 1pt}c@{\hskip 1pt} l r}
    \toprule
    \textbf{N} & \multicolumn{3}{c}{\textbf{FID $\downarrow$}} & \multicolumn{3}{c}{\textbf{IS $\uparrow$}} & \textbf{Violation $\downarrow$} \\
    \midrule
    1 & 7.977 & \makebox[9pt][c]{} &  & 8.608 & \makebox[9pt][c]{$\!\!\pm\!\!$} & 0.086 & 0 \\
    2 & 7.823 & \makebox[9pt][c]{} &  & 8.570 & \makebox[9pt][c]{$\!\!\pm\!\!$} & 0.069 & 0 \\
    3 & 7.972 & \makebox[9pt][c]{} &  & 8.585 & \makebox[9pt][c]{$\!\!\pm\!\!$} & 0.118 & 0 \\
    4 & 8.007 & \makebox[9pt][c]{} &  & 8.671 & \makebox[9pt][c]{$\!\!\pm\!\!$} & 0.087 & 0 \\
    5 & 8.004 & \makebox[9pt][c]{} &  & 8.591 & \makebox[9pt][c]{$\!\!\pm\!\!$} & 0.104 & 0 \\
    6 & 7.975 & \makebox[9pt][c]{} &  & 8.578 & \makebox[9pt][c]{$\!\!\pm\!\!$} & 0.081 & 0 \\
    \bottomrule
    \end{tabular}
    \end{adjustbox}
    \label{tab:num_constraint_layer}
\end{table}

\section{Additional Experimental Details for Partial Charge Predictions}
\paragraph{Training} Here, we describe our training and evaluation process for the exact-k constrained MPNN. We conducted a random partitioning of the dataset containing 2266 charge-labeled MOFs, creating distinct training, validation, and test sets (70/10/20\%). We use the training set for direct model parameter tuning, while the validation set determines stopping criteria. The test set plays a crucial role in providing an unbiased assessment of the final model's performance.

\paragraph{Hyperparameter Tuning} To optimize our model's performance, we conduct a systematic hyperparameter tuning process, sequentially optimizing six key hyperparameters: Learning rate, Batch size, Time steps, Embedding size, Hidden Feature size, and Patience Threshold. The optimal hyperparameter values are reported in supplementary materials.

The optimal hyperparameter values for closed-form expected loss are: lr = 0.005, batch size = 128, time steps = 6, embedding size = 20, hidden feature size = 50, and patience threshold = 300, and the optimal hyperparameter values for negative log likelihood loss are : lr = 0.005, batch size = 128, time steps = 5, embedding size = 30, hidden feature size = 50, and patience threshold = 300.

\section{Additional Experimental Details for Chemical Process Units and Subsystems}
\label{appendix:chemical_process_units_and_subsystems}
\textbf{Experiment setting} In a Continuous Stirred-Tank Reactor (\textbf{CSTR}) benzene and ethylene react to produce ethylbenzene following $B + E \rightarrow EB$. The stoichiometric relationship ensures reactant consumption matches product formation, governed by linear equality constraints. We consider the setting where the reactor operates with fixed volume and pressure, leaving the molar flow rates of benzene, ethylene, and the working temperature as design variables. A neural surrogate model is trained to predict output flow rates. In a chemical plant (\textbf{plant}) that produces DME and DEE uses methanol, ethanol, and water as feed, Mass balance ensures consistency between inflows, reactions, outflows, and recycling streams. Assuming the entire system is fixed, a surrogate model can predict output flow based on inputs and recycling. In an extractive distillation subsystem (\textbf{distillation}) which separates a 50/50 azeotropic mixture of n-heptane and toluene using phenol as a solvent, two distillation columns separate n-heptane at the top and toluene with phenol solvent recovered at the bottom. A surrogate model is developed to predict heat duties and the flow rates of components in the distillate streams for optimal operating conditions. With no chemical reactions, the sum of the molar flow rates of n-heptane, toluene, and phenol always equals the distillate rate.

We specify the input and output variables as well as governing constraints for the three experiments. $x_i$ denotes the input variables and $y_i$ denotes the output variables.
\paragraph{CSTR}
\begin{itemize}
    \item $x_1$: Temperature of the RX unit.
    \item $x_2$: Molar flow rate of Benzene in the B stream.
    \item $x_3$: Molar flow rate of Ethylene in the E stream.
    \item $y_1$: Molar flow rate of Ethylbenzene in the EB stream.
    \item $y_2$: Molar flow rate of Benzene in the EB stream.
    \item $y_3$: Molar flow rate of Ethylene in the EB stream.
\end{itemize}
The linear equality constraints governing the system are:
\[
 - y_2 + y_3 = - x_2 + x_3, \quad \text{Reactants consumption},
\]
\[
 - y_1 - y_2 = - x_2, \quad \text{EB production}.
\]

\paragraph{plant}
\begin{itemize}
    \item $x_1$: Mass flow rate of methanol in the FEED stream.
    \item $x_2$: Mass flow rate of ethanol in the FEED stream.
    \item $x_3$: Mass flow rate of water in the FEED stream.
    \item $x_4$: Total mass flow rate of the PURGE stream.
    \item $y_1$: Total mass flow rate of the DME stream.
    \item $y_2$: Mass flow rate of DME in the DME stream.
    \item $y_3$: Total mass flow rate of the DEE stream.
    \item $y_4$: Mass flow rate of DEE in the DEE stream.
    \item $y_5$: Total mass flow rate of the WATER stream.
\end{itemize}

The system satisfies the following mass balance equation:
\[
- y_1 - y_3 - y_5 = -x_1 - x_2 - x_3 + x_4, \quad \text{Mass balance}.
\]

\paragraph{distillation}
\begin{itemize}
    \item $x_1$: Molar flow rate of phenol in the SOLVENT stream.
    \item $x_2$: Reflux ratio of COLUMN column.
    \item $x_3$: Distillate rate of COLUMN column.
    \item $x_4$: Reflux ratio of COL-REC column.
    \item $x_5$: Distillate rate of COL-REC column.
    \item $y_1$: Molar flow rate of $n$-heptane in the C7 stream.
    \item $y_2$: Molar flow rate of toluene in the TOLUENE stream.
    \item $y_3$: Condenser heat duty of COLUMN column.
    \item $y_4$: Reboiler heat duty of COLUMN column.
    \item $y_5$: Condenser heat duty of COL-REC column.
    \item $y_6$: Reboiler heat duty of COL-REC column.
    \item $y_7$: Molar flow rate of toluene in the C7 stream.
    \item $y_8$: Molar flow rate of phenol in the C7 stream.
    \item $y_9$: Molar flow rate of $n$-heptane in the TOLUENE stream.
    \item $y_{10}$: Molar flow rate of phenol in the TOLUENE stream.
\end{itemize}

The system satisfies the following linear equality constraints:
\[
- y_1 - y_7 - y_8 = -x_3, \quad \text{C7 fractions},
\]
\[
- y_2 - y_9 - y_{10} = -x_5, \quad \text{TOLUENE fractions}.
\]

\paragraph{Training and Validation Loss Curve}
We present the training and validation loss curve, which demonstrates the faster convergence of our approach.

\begin{figure}[t]
    \centering
    \includegraphics[width=0.45\textwidth]{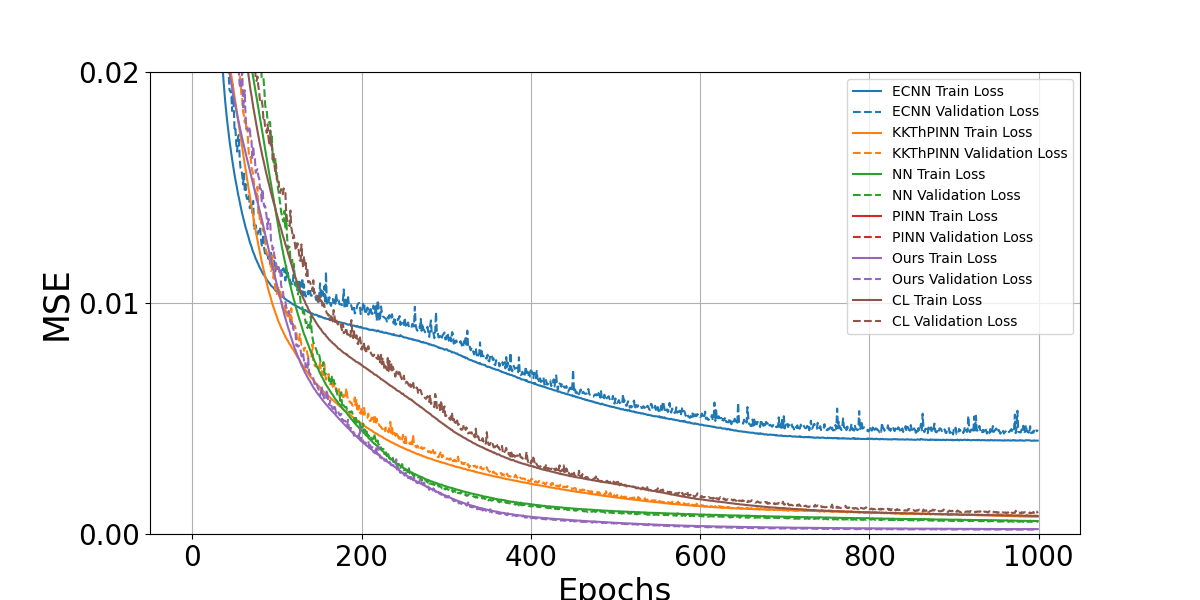}
    \caption{Training and validation MSE loss curve for \textbf{CSTR}. All results are averaged over 10 independent runs.}
    \label{fig:cstr_train_valid}
\end{figure}

\begin{figure}[t]
    \centering
    \includegraphics[width=0.45\textwidth]{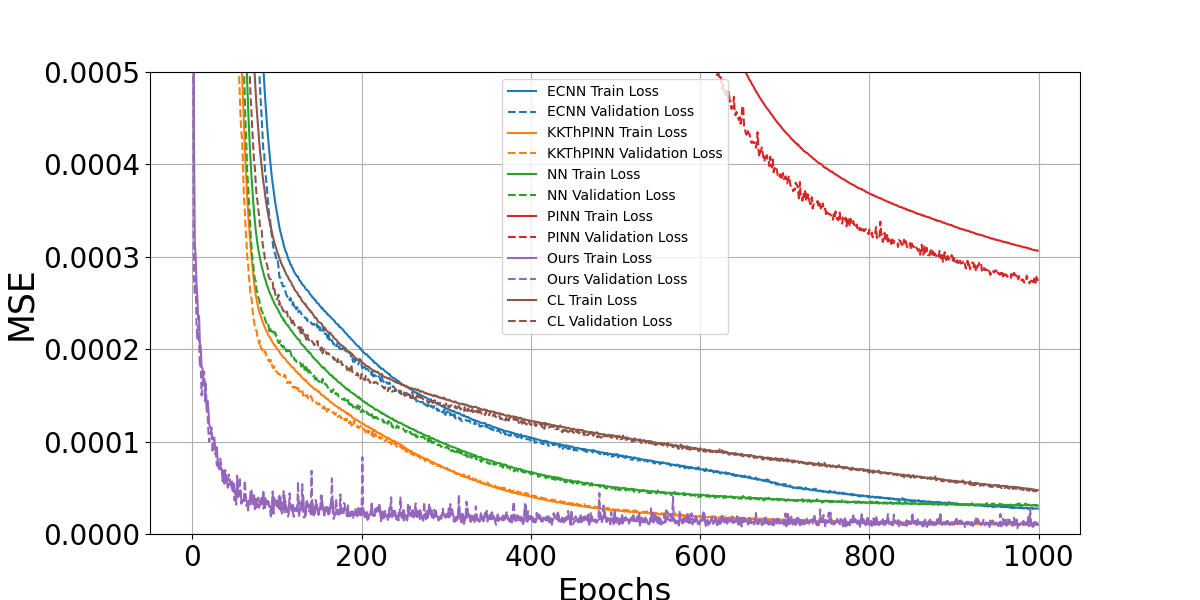}
    \caption{Training and validation MSE loss curve for \textbf{plant}. All results are averaged over 10 independent runs.}
    \label{fig:plant_train_valid}
\end{figure}

\begin{figure}[t]
    \centering
    \includegraphics[width=0.45\textwidth]{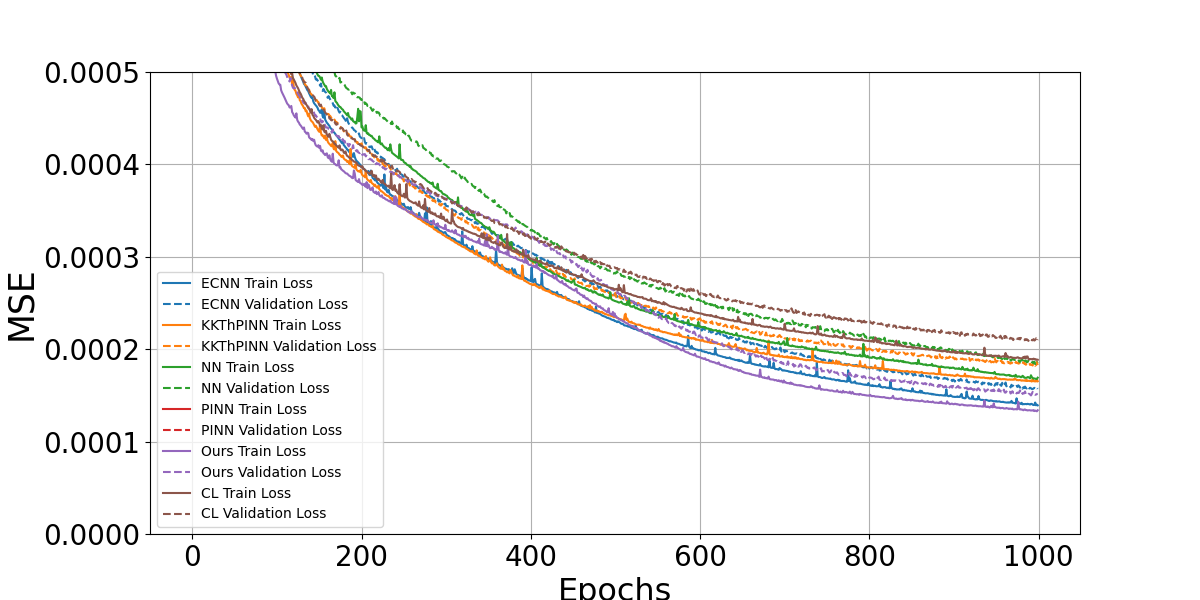}
    \caption{Training and validation MSE loss curve for \textbf{distillation}. All results are averaged over 10 independent runs.}
    \label{fig:distillation_train_valid}
\end{figure}



\section{Beyond Gaussian}
\label{sec:Beyond Gaussian}
In this section, we present the theoretical results 
when $\z$ are Poisson variables defined over discrete domains.
Similar to the Gaussian setting, we find that when the element-wise loss $\ell$ is $L1$ or $L2$ loss and the mapping $\decoder$ is an identity function, the expected loss admits closed-form expressions.
\begin{proposition}[Poisson Closed-form Expected Loss]\label{prop: closed form poisson}
    Let $\z = (z_1, \ldots, z_n)^T$, where $z_i \sim Poisson(\theta_i)$. Let $\bm{y} = \left( y_1, y_2, \ldots, y_n \right)^T$ be the ground truth vector subject to the equality constraint $\sum_{i=1}^n y_i = k$ with $k \in \mathbb{N}^+$. Then it holds that
    \begin{enumerate}[label=\roman*),noitemsep,topsep=0pt]
        \item when $\ell$ is L1 loss,
        \begin{align*}
        L(\btheta) 
        & =
        \sum_{i=1}^n (k - \floor{k p_i - d_i}) p_i Bin(\floor{k p_i - d_i}; k, p_i) \\
        & + \floor{k p_i - d_i} (1-p_i) Bin(\floor{k p_i - d_i}; k, p_i) \\
        & - 2 d_i F(\floor{k p_i - d_i}; k, p_i) + d_i;
    \end{align*}
        \item when $\ell$ is L2 loss,
        $
        L(\btheta) 
        =
        \sum_{i=1}^n
            k
            p_i
            \left( 1 - p_i \right)
            +
            k^2
            p_i^2
            -
            2 y_i k p_i
            +
            y_i^2 ,
    $
    \end{enumerate}
    where $Bin$ denotes the probability mass function~(p.m.f.) of a binomial distribution and $F$ denotes a regularized incomplete beta function. $d_i = k p_i - y_i$ and $p_i =\frac{\theta_i}{\sum_{j=1}^n \theta_j}$.
\end{proposition}

In the general setting when there is no closed-form solution for the expected loss, we first show that exact sampling from the constrained distribution can be achieved as it takes its form as a multinomial distribution.
\begin{proposition}[Poisson Constrained Distribution]\label{prop: Poisson Constrained Distribution}
    Given $\z = (z_1, \ldots, z_n)^T$ with $z_i \sim Poisson(\theta_i)$, the constrained distribution $p_{\logits}(\bm{z} \mid \sum_{j=1}^n z_n = k)$ is equivalent to a multinomial distribution with parameter $k$ and probabilities $\frac{\theta_1}{\sum_{j=1}^n \theta_j}, \ldots, \frac{\theta_n}{\sum_{j=1}^n \theta_j}$.
\end{proposition}
In the backward pass, the results below allow us to derive gradient estimators with either the conditional marginals or the expectation of the conditional marginal as a differentiable proxy. 

\begin{proposition}[Poisson Conditional Marginal and Expectations]\label{prop: Poisson Conditional Marginal}
    Given $\z = (z_1, \ldots, z_n)^T$ with $z_i \sim Poisson(\theta_i)$, the conditional marginal $p_{\logits}(z_i \mid \sum_{j=1}^n z_j = k)$ follows a binomial distribution with parameter $k$ and probability $\frac{\theta_i}{\sum_{j=1}^n \theta_j}$. Further, its expectation 
    is $\frac{k\theta_i}{\sum_{j=1}^n \theta_j}$.
\end{proposition}

We conduct Synthetic Experiment for Poisson variables with settings similar to those of Gaussian variables. We compare our gradient estimator \textit{Marginal Expectation} with three baselines, namely \textit{Random}, \textit{Unconstrained Marginal}, as well as \textit{Constrained Marginal}. Results are shown in ~\ref{fig:SyntheticPoisson}. We observe similar results as in the Gaussian settings. Estimator \textit{Unconstrained Marginal} and \textit{Random} have similar performances. \textit{Constrained Marginal} still has similarly bad performances as \textit{Random}, while \textit{Marginal Expectation} outperforms these baselines by a noticeable margin. We show that our \textit{Marginal Expectation} is not limited to Gaussian and Bernoulli variables.
\begin{figure}[t]
    \centering
    \begin{subfigure}
        \centering
        \includegraphics[height=0.125\textheight]{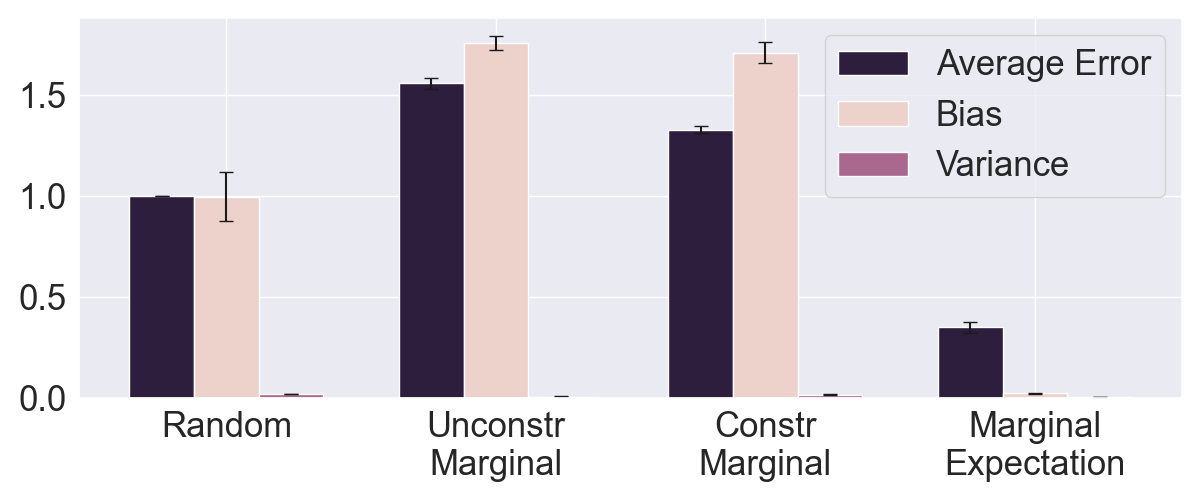}
    \end{subfigure}
    \hspace{0.08\textwidth}
    \begin{subfigure}
        \centering
        \includegraphics[height=0.125\textheight]{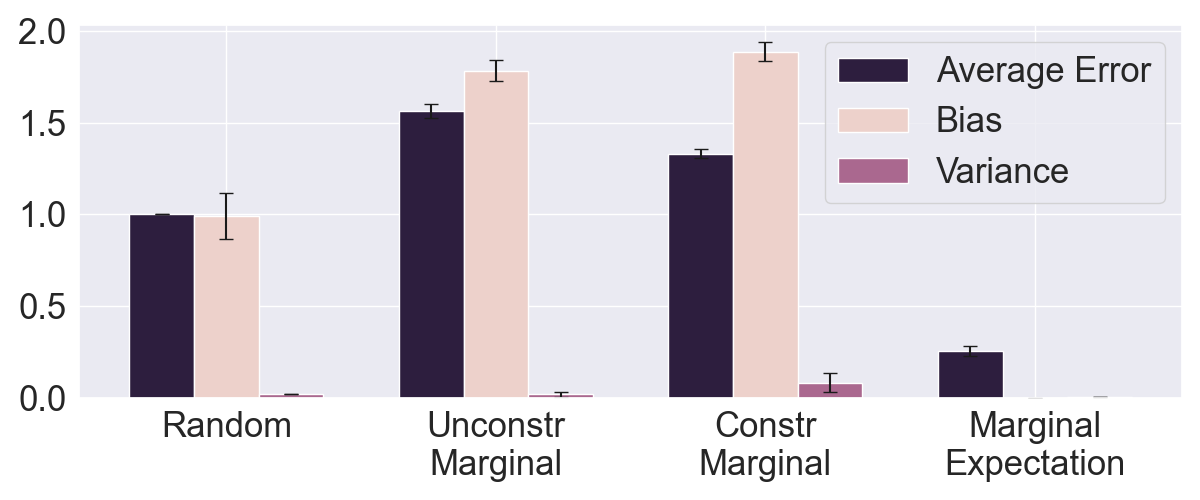}
    \end{subfigure}
    \caption{Comparisons of different gradient estimators for point-wise loss $\ell$ being L1 and L2 loss applied to Poisson variable are conducted. The comparison follows the same setting as Gaussian variables. We show that our proposed gradient estimator \textit{Marginal Expectation} outperforms baselines by a significant margin.}
    \label{fig:SyntheticPoisson}
\end{figure}

\section{Proofs}
\subsection{Proposition~\ref{prop: Gaussian Constrained Distribution}}
\begin{proposition}[Gaussian Constrained Distribution]
\label{prop: Gaussian Constrained Distribution} 
Given $\z = \left( z_1, \ldots, z_n \right)^T \sim \mathcal{N} \left( \bm{\mu}, \bm{\Sigma} \right)$, 
the constrained distribution $p_{\logits}(\z \mid \bm{A} \bm{z} = \bm{k})$ is equivalent to an $n-a$ dimensional multivariate Gaussian distribution with mean $\overline{\bm{\mu}} \in \R^{n - a}$ and covariance matrix $\overline{\bm{\Sigma}} \in \R^{n - a \times {n - a}} $ defined as
\begin{align*}
    &\overline{\bm{\mu}} = 
        \bm{E} \bm{\mu} + \bm{E} \bm{\Sigma} \bm{A}^T \left( \bm{A} \bm{\Sigma} \bm{A}^T \right)^{-1} \left( \bm{k} - \bm{A} \bm{\mu} \right), \\
     &\overline{\bm{\Sigma}} =
        \bm{E} \bm{\Sigma} \bm{E}^T - \bm{E} \bm{\Sigma} \bm{A}^T \left( \bm{A} \bm{\Sigma} \bm{A}^T \right)^{-1} \bm{A} \bm{\Sigma} \bm{E}^T.
\end{align*}
where $\bm{E} \in \R^{(n - a) \times (n)}$ is the first $n-a$ rows of an identity matrix $\bm{I} \in \R^{n \times n}$.
\end{proposition}

\begin{proof}
\label{proof:SystemofLinearEqualityConstrainedDistribution}
Let $\z = (z_1, \ldots, z_n)^T$ with $\z \sim \mathcal{N} (\bm{\mu}, \bm{\Sigma})$. Define $\bm{A}$ be the constraint matrix with $\text{rank} (\bm{A}) = a < n$. Define $\bm{E} \in \R^{(n - a) \times (n)}$ to contain the first $n-a$ rows of an identity matrix $\bm{I} \in \R^{n \times n}$. Thus, we consider the following linear transformation defined by $\bm{C}$.
\begin{align*}
    \bm{C} \bm{\mu} & =
        \begin{pmatrix}
            \bm{E} \bm{\mu} \\
            \bm{A} \bm{\mu} \\
        \end{pmatrix} \\
    \bm{C} \bm{\Sigma} \bm{C}^T & =
        \begin{pmatrix}
            \bm{E} \bm{\Sigma} \bm{E}^T & \bm{E} \bm{\Sigma} \bm{A}^T\\
            \bm{A} \bm{\Sigma} \bm{E}^T & \bm{A} \bm{\Sigma} \bm{A}^T\\
        \end{pmatrix} \\
\end{align*}
Thus, the conditional distribution of $\bm{E} \bm{z}$ condition on $\bm{A} \bm{z}$ follows multivariate Gaussian distribution with parameters as follows:
\begin{align*}
    \overline{\bm{\mu}} & = 
        \bm{E} \bm{\mu} + \bm{E} \bm{\Sigma} \bm{A}^T \left( \bm{A} \bm{\Sigma} \bm{A}^T \right)^{-1} \left( \bm{k} - \bm{A} \bm{\mu} \right) \\
    \overline{\bm{\Sigma}} & =
        \bm{E} \bm{\Sigma} \bm{E}^T - \bm{E} \bm{\Sigma} \bm{A}^T \left( \bm{A} \bm{\Sigma} \bm{A}^T \right)^{-1} \bm{A} \bm{\Sigma} \bm{E}^T
\end{align*}
\end{proof}

\subsection{Proposition~\ref{prop:closed form gaussian}}
\label{proof:Closed-formExpectedLossunderGaussian}
\begin{proposition}[Gaussian Closed-form Expected Loss]\label{prop:closed form gaussian} 
    Let $\z \sim \mathcal{N} \left( \bm{\mu}, \bm{\Sigma} \right)$. 
    Let $\bm{y} = \left( y_1, \ldots, y_n \right)^T$ be the ground truth vector subject to the equality constraint $\bm{A} \bm{z} = \bm{k}$. 
    Then it holds that
    \begin{enumerate}[label=\roman*),noitemsep,topsep=0pt]
        \item when $\ell$ is L1 loss,
        $L(\btheta)$ has closed form 
        
        $
        \sum_{i=1}^n 
\overline{\bm{\Sigma}}_{i,i} \sqrt{\frac{2}{\pi}} 
e^{-\frac{(\overline{\bm{\mu}}_i - y_i)^2}{2 \overline{\bm{\Sigma}}_{i,i}^2}}
+ 
(\overline{\bm{\mu}}_i - y_i) \,\textit{erf} \left( \frac{\overline{\bm{\mu}}_i - y_i}{\sqrt{2} \, \overline{\bm{\Sigma}}_{i,i}} \right)
    $;
        \item when $\ell$ is L2 loss,
        $
        \sum_{i=1}^n
            \overline{\bm{\mu}}_i^2 + \overline{\bm{\Sigma}}_{i,i}^2
            -
            2 y_i \overline{\bm{\mu}}_i
            +
            y_i^2$,
    \end{enumerate}
    where $\overline{\bm{\mu}}$ and $\overline{\bm{\Sigma}}$ are defined above.
\end{proposition}

\begin{proof}
Let $\z = \left( z_1, \ldots, z_n \right)^T \sim \mathcal{N} \left( \bm{\mu}, \bm{\Sigma} \right)$. Let $\bm{y} = \left( y_1, y_2, \ldots, y_n \right)^T$ be the ground truth subject to the equality constraint $\bm{A} \bm{y} = \bm{k}$. We derive a closed-form solution for the L1 loss of $\z$ subject to the constraint $\bm{A} \bm{z} = \bm{k}$.
\begin{align*}
    L(\theta)
    & = \mathds{E}_{\bm{z} \sim p_{\theta}(\bm{z} \mid \bm{A} \bm{z} = \bm{k})}[\parallel \bm{z} - \bm{y} \parallel] \\
    & = \sum_{i=1}^n \mathds{E}_{\bm{z} \sim p_{\theta}(\bm{z} \mid \bm{A} \bm{z} = \bm{k})}[\parallel z_i - y_i \parallel]
\end{align*}
We know that $z_i \sim N \left( \overline{\mu}_i, \overline{\sigma}_i^2 \right)$ from below. Then, define $l_i = z_i - y_i$. Then $l_i \sim N \left( \overline{\mu}_i - y_i, \overline{\sigma}_i^2 \right)$. Thus, $\mathds{E}_{\bm{z} \sim p_{\theta}(\bm{z} \mid \bm{A} \bm{z} = \bm{k})}[\mid l_i \mid]$ is the mean of a folded normal distribution.
\begin{align*}
    \mathds{E}_{\bm{z} \sim p_{\theta}(\bm{z} \mid \bm{A} \bm{z} = \bm{k})}[\mid l_i \mid]
    =
    \overline{\sigma}_i^2 \sqrt{\frac{2}{\pi}} \exp \left( \frac{- (\overline{\mu}_i - y_i)^2}{2 \overline{\sigma}_i^2} \right) \\
    +
    (\overline{\mu}_i - y_i) erf \left( \frac{\overline{\mu}_i - y_i}{\sqrt{2 \overline{\sigma}_i^2}} \right)
\end{align*}
We also derive a closed-form solution for the L2 loss of $\z$ subject to the constraint $\bm{A} \bm{z} = \bm{k}$.
\begin{align*}
    & L(\theta)
    = \mathds{E}_{\bm{z} \sim p_{\theta}(\bm{z} \mid \bm{A} \bm{z} = \bm{k})}[\parallel \bm{z} - \bm{y}\parallel^2] \\
    & = \sum_{i=1}^n \mathds{E}_{\bm{z} \sim p_{\theta}(z_i \mid \bm{A} \bm{z} = \bm{k})}[\parallel z_i - y_i \parallel^2] \\
    & = \sum_{i=1}^n \mathds{E}_{\bm{z} \sim p_{\theta}(z_i \mid \bm{A} \bm{z} = \bm{k})}[z_i^2 - 2y_iz_i + y_i^2] \\
    & = \sum_{i=1}^n \overline{\mu}_i + \overline{\sigma}_i^4 - 2 y_i \overline{\mu}_i + y_i^2  
\end{align*}
\end{proof}

\subsection{Proposition~\ref{prop: Gaussian Conditional Marginal}}

\begin{proposition}[Gaussian Conditional Marginal and Expectations]\label{prop: Gaussian Conditional Marginal}
    Given $\z = \left( z_1, \ldots, z_n \right)^T \sim \mathcal{N} \left( \bm{\mu}, \bm{\Sigma} \right)$, the conditional marginal $p_{\logits}(z_i \mid \bm{A} \z = \bm{k})$ follows a univariate Gaussian distribution with mean $\overline{\mu}_i = \mu_i + \bm{e}_i^T \bm{\Sigma} \bm{A} \left( \bm{A} \bm{\Sigma} \bm{A}^T \right)^{-1} \left( \bm{k} - \bm{A} \bm{\mu} \right)$ and variance $\overline{\sigma}_i^2 = \bm{e}_i^T \bm{\Sigma} \bm{e}_i - \bm{e}_i^T \bm{\Sigma} \bm{A}^T \left( \bm{A} \bm{\Sigma} \bm{A}^T \right)^{-1} \bm{A} \bm{\Sigma} \bm{e}_i$.
    Further, the expectation of the marginal distribution is $\overline{\mu}_i$.
\end{proposition}

\begin{proof}
To derive the marginal distribution, let's consider the standard basis vector $\bm{e}_i$ in $\R^{n}$. Define the linear transformation $\bm{B}$ such that
\begin{align*}
    \bm{B} = \begin{pmatrix}
        \bm{e}_i^T \\
        \bm{A} \\
    \end{pmatrix}
\end{align*}
Then, 
\begin{align*}
    \bm{B} \bm{\mu} & =
        \begin{pmatrix}
            \mu_i \\
            \bm{A} \bm{\mu} \\
        \end{pmatrix} \\
    \bm{B} \bm{\Sigma} \bm{B}^T & =
        \begin{pmatrix}
            \bm{e}_i^T \bm{\Sigma} \bm{e}_i & \bm{e}_i^T \bm{\Sigma} \bm{A}^T\\
            \bm{A} \bm{\Sigma} \bm{e}_i & \bm{A} \bm{\Sigma} \bm{A}^T\\
        \end{pmatrix} \\
\end{align*}
Thus, the conditional distribution of $z_i$ condition on $\bm{A} \bm{z}$ follows multivariate normal distribution with parameters as follows:
\begin{align*}
    \overline{\mu}_i & = 
        \mu_i + \bm{e}_i^T \bm{\Sigma} \bm{A} \left( \bm{A} \bm{\Sigma} \bm{A}^T \right)^{-1} \left( \bm{k} - \bm{A} \bm{\mu} \right) \\
    \overline{\sigma}_i^2 & =
        \bm{e}_i^T \bm{\Sigma} \bm{e}_i - \bm{e}_i^T \bm{\Sigma} \bm{A}^T \left( \bm{A} \bm{\Sigma} \bm{A}^T \right)^{-1} \bm{A} \bm{\Sigma} \bm{e}_i
\end{align*}
\end{proof}

\subsection{Proposition~\ref{prop: closed form poisson}}
\label{proof:Closed-formExpectedLossunderPoisson}

\begin{proposition}[Poisson Closed-form Expected Loss]\label{prop: closed form poisson}
    Let $\z = (z_1, \ldots, z_n)^T$, where $z_i \sim Poisson(\theta_i)$. Let $\bm{y} = \left( y_1, y_2, \ldots, y_n \right)^T$ be the ground truth vector subject to the equality constraint $\sum_{i=1}^n y_i = k$ with $k \in \mathbb{N}^+$. Then it holds that
    \begin{enumerate}[label=\roman*),noitemsep,topsep=0pt]
        \item when $\ell$ is L1 loss,
        \begin{align*}
        L(\btheta) 
        & =
        \sum_{i=1}^n (k - \floor{k p_i - d_i}) p_i Bin(\floor{k p_i - d_i}; k, p_i) \\
        & + \floor{k p_i - d_i} (1-p_i) Bin(\floor{k p_i - d_i}; k, p_i) \\
        & - 2 d_i F(\floor{k p_i - d_i}; k, p_i) + d_i;
    \end{align*}
        \item when $\ell$ is L2 loss,
        $
        L(\btheta) 
        =
        \sum_{i=1}^n
            k
            p_i
            \left( 1 - p_i \right)
            +
            k^2
            p_i^2
            -
            2 y_i k p_i
            +
            y_i^2 ,
    $
    \end{enumerate}
    where $Bin$ denotes the probability mass function~(p.m.f.) of a binomial distribution and $F$ denotes a regularized incomplete beta function. $d_i = k p_i - y_i$ and $p_i =\frac{\theta_i}{\sum_{j=1}^n \theta_j}$.
\end{proposition}

\begin{proof}
Let $\z = (z_1, \ldots, z_n)^T$, where $z_i \sim Poisson(\theta_i)$. Let $\bm{y} = \left( y_1, y_2, \ldots, y_n \right)^T$ be the ground truth vector subject to the equality constraint $\sum_{j=1}^n y_j = k$. We derive the closed-form expression for the L1 loss of $\z$ subject to the constraint $\sum_{j=1}^n z_j = k$.

\begin{align*}
    L(\btheta) & = \mathds{E}_{\bm{z} \sim p(\bm{z} \mid \sum_{j=1}^n z_j = k)}[\parallel \bm{z} - \bm{y} \parallel_1] \\
    & = \sum_{i=1}^n \mathds{E}_{z_i \sim p(z_i \mid \sum_{j=1}^n z_j = k)}[| z_i - y_i |] \\
\end{align*}
Define $d_i = k p_i - y_i$, where $p_i =\frac{\theta_i}{\sum_{j=1}^n \theta_j}$. Then,
\begin{align*}
    L(\btheta) & = \sum_{i=1}^n \mathds{E}_{z_i \sim p(z_i \mid \sum_{j=1}^n z_j = k)}[| z_i - k p_i + d_i|] \\
    & = \sum_{i=1}^n \sum_{all\ z_i} | z_i - k p_i + d_i| Binomial(z_i; k, p_i) \\
    = & \sum_{z_i = 0}^{\floor{kp_i - d_i}} ( - z_i + k p_i ) Binomial(z_i; k, p_i) \\
    + & \sum_{\floor{k p_i - d_i}}^{k} ( z_i - k p_i ) Binomial(z_i; k, p_i) \\
    - & d_i \sum_{z_i = 0}^{\floor{k p_i - d_i}} Binomial(z_i; k, p_i) \\
    + & d_i   \sum_{\floor{k p_i - d_i}}^{k} Binomial(z_i; k, p_i)
\end{align*}
Consider the following lemma (Todhunter's Formula \citet{diaconis1991closed})
\begin{lemma}
For all integers $0 \leq \alpha \leq \beta \leq n$,
\begin{align*}
    & \sum_{x=\alpha}^{\beta} (x-np) Binomial(x;n,p) \\
    = & \alpha (1-p) Binomial(\alpha; n, p) \\
    - & (n-\beta) p Binomial(\beta; n, p)
\end{align*}
\end{lemma}
Then,
\begin{align*}
    & \sum_{z_i = 0}^{\floor{kp_i - d_i}} ( - z_i + k p_i ) Binomial(z_i; k, p_i) \\
    & = (k - \floor{k p_i - d_i}) p_i Binomial(\floor{k p_i - d_i}; k, p_i)
\end{align*}
and
\begin{align*}
    & \sum_{\floor{k p_i - d_i}}^{k} ( z_i - k p_i ) Binomial(z_i; k, p_i) \\
    & = \floor{k p_i - d_i} (1-p_i) Binomial(\floor{k p_i - d_i}; k, p_i)
\end{align*}
Next, we notice that
\begin{align*}
    & -  d_i \sum_{z_i = 0}^{\floor{k p_i - d_i}} Binomial(z_i; k, p_i) \\
    & + d_i   \sum_{\floor{k p_i - d_i}}^{k} Binomial(z_i; k, p_i) \\
    = & -  d_i \sum_{z_i = 0}^{\floor{k p_i - d_i}} Binomial(z_i; k, p_i) \\
    + & d_i \left( 1 - \sum_{z_i = 0}^{\floor{k p_i - d_i}} Binomial(z_i; k, p_i) \right) \\
    = & - 2 d_i \sum_{z_i = 0}^{\floor{k p_i - d_i}} Binomial(z_i; k, p_i) + d_i
\end{align*}
Define the regularized incomplete beta function as
\begin{equation*}
    F(x; n, p) = (n-x) {n \choose x} \int_{0}^{1-p} t^{n-x-1} (1-t)^x\ dt
\end{equation*}
Then,
\begin{align*}
    & - 2 d_i \sum_{z_i = 0}^{\floor{k p_i - d_i}} Binomial(z_i; k, p_i) + d_i \\
    = &
    -2 d_i F(\floor{k p_i - d_i}; k, p_i) + d_i 
\end{align*}
Thus, the closed-form expression for the L1 loss is
\begin{align*}
    & L(\btheta) \\
    = & 
    \sum_{i=1}^n (k - \floor{k p_i - d_i}) p_i Binomial(\floor{k p_i - d_i}; k, p_i) \\
    + & \floor{k p_i - d_i} (1-p_i) Binomial(\floor{k p_i - d_i}; k, p_i) \\
    - & 2 d_i F(\floor{k p_i - d_i}; k, p_i) + d_i 
\end{align*}
We attempt to derive a closed-form solution for the L2 loss of $\z$ subject to the constraint $\sum_{j=1}^n z_j = k$.
\begin{align*}
    L(\btheta)
    & = \mathds{E}_{\bm{z} \sim p_{\btheta}(\bm{z} \mid \sum_i z_i = k)}[\parallel \bm{z} - \bm{b }\parallel_2^2] \\
    & =
    \sum_{i=1}^n \mathds{E}_{z_i \sim p_{\btheta}(z_i \mid \sum_j z_j = k)} [z_i^2] \\
    & -
    2
    \sum_{i=1}^n
    y_i
    \mathds{E}_{z_i \sim p_{\btheta}(z_i \mid \sum_j z_j = k)} [z_i]
    +
    \sum_{i=1}^n y_i^2
\end{align*}
Since the conditional marginal distribution is a binomial distribution, it's second moment is given by
\begin{align*}
    & \sum_{i=1}^n \mathds{E}_{z_i \sim p_{\btheta}(z_i \mid \sum_j z_j = k)} [z_i^2] \\
    = &
    \sum_{i=1}^n
        k
        \left( \frac{\theta_i}{\sum_{j=1}^n \theta_j} \right)
        \left( \frac{\sum_{j=1}^n \theta_j - \theta_i}{\sum_{j=1}^n \theta_j} \right)
        +
        \left( \frac{k \theta_i}{\sum_{j=1}^n \theta_j} \right)^2
\end{align*}
It's first moment(mean) is given by
\begin{equation*}
    -
    2
    \sum_{i=1}^n
    y_i
    \mathds{E}_{z_i \sim p_{\btheta}(z_i \mid \sum_j z_j = k)} [z_i]
    =
    -
    2 k
    \sum_{i=1}^n
    y_i \left( \frac{\theta_i}{\sum_{j=1}^n \theta_j} \right)
\end{equation*}
Thus, we have
\begin{align*}
    & \sum_{i=1}^n \mathds{E}_{z_i \sim p_{\btheta}(z_i \mid \sum_j z_j = k)} [z_i^2] \\
    = &
    \sum_{i=1}^n
        k
        \left( \frac{\theta_i}{\sum_{j=1}^n \theta_j} \right)
        \left( \frac{\sum_{j=1}^n \theta_j - \theta_i}{\sum_{j=1}^n \theta_j} \right)
        +
        \left( \frac{k \theta_i}{\sum_{j=1}^n \theta_j} \right)^2
    \\
    & -
    2 k
    \sum_{i=1}^n
    y_i \left( \frac{\theta_i}{\sum_{j=1}^n \theta_j} \right)
    +
    \sum_{i=1}^n y_i^2
\end{align*}
\end{proof}

\subsection{Proposition~\ref{prop: Poisson Constrained Distribution}}
\label{proof:PoissonConstrainedDistribution}

\begin{proposition}[Poisson Constrained Distribution]\label{prop: Poisson Constrained Distribution}
    Given $\z = (z_1, \ldots, z_n)^T$ with $z_i \sim Poisson(\theta_i)$, the constrained distribution $p_{\logits}(\bm{z} \mid \sum_{j=1}^n z_n = k)$ is equivalent to a multinomial distribution with parameter $k$ and probabilities $\frac{\theta_1}{\sum_{j=1}^n \theta_j}, \ldots, \frac{\theta_n}{\sum_{j=1}^n \theta_j}$.
\end{proposition}

\begin{proof}
Let $\z = (z_1, \ldots, z_n)^T$, where $z_i \sim Poisson(\theta_i)$. We compute a closed-form solution for the conditional probability $p_{\logits} \left( \z \mid \sum_{j=1}^n z_j = k \right)$.
\begin{equation*}
    p_{\logits} \left(( \z | \sum_{j=1}^n z_j = k \right) = \frac{p \left( \z \cap \sum z_i = k \right)}{p \left( \sum_{j=1}^n z_j = k \right)}
\end{equation*}
Let $Y = \sum_{j=1}^n z_j$. The denominator is the p.d.f. of $Y$ evaluated at $k$. Since $Y$ is a linear combination of independent Poisson random variables, we know $Y \sim Poisson(\sum_{j=1}^n \theta_j)$. Thus,
\begin{equation*}
    p \left( \sum_{j=1}^n z_j = k \right)
    =
    \frac{e^{- \sum_{j=1}^n \theta_j} \left( \sum_{j=1}^n \theta_j \right)^k}{k!}
\end{equation*}
Next, let's consider the numerator.
\begin{equation*}
    p(\bm{z} \cap \sum_{j=1}^n z_j = k) = 
    \begin{cases} 
      p(\bm{z}) & \sum_{j=1}^n z_j = k \\
      0 &\sum_{j=1}^n z_j \neq k
   \end{cases}
\end{equation*}
where $p(\z) = \prod_{i=1}^n f(z_i) = \prod_{i=1}^n \frac{e^{-\theta_i} \theta_i^{z_i}}{z_i!}$.
Thus, our conditional distribution is given by
\begin{align*}
    & p(\bm{z} | \sum_{j=1}^n z_j = k) \\
    = &
    \begin{cases} 
      \frac{\frac{e^{- \sum_{i=1}^n \theta_i} \prod_{i=1}^n \theta_i^{z_i}}{\prod_{i=1}^n z_i !}}
      {\frac{e^{- \sum_{i=1}^n \theta_i} \left( \sum_{i=1}^n \theta_i \right)^k}{k!}} 
      & \sum_{j=1}^n z_j = k \\ 
      0 & \sum_{j=1}^n z_j \neq k
   \end{cases} \\
   = &
    \begin{cases} 
      \frac{k! \prod_{i=1}^n \theta_i^{z_i} }
      {(\sum_{i=1}^n \theta_i)^k \prod_{i=1}^n z_i !} 
      & \sum_{j=1}^n z_j = k \\ 
      0 & \sum_{j=1}^n z_j \neq k
   \end{cases} \\
   = &
    \begin{cases} 
       \frac{1}{(\sum_{i=1}^n \theta_i)^k} \cdot \frac{k!}{\prod_{i=1}^n z_i !} \prod_{i=1}^n \theta_i^{z_i} 
      & \sum_{j=1}^n z_j = k \\ 
      0 & \sum_{j=1}^n z_j \neq k
   \end{cases} \\
   = &
    \begin{cases} 
      \frac{k!}{\prod_{i=1}^n z_i !} \prod_{i=1}^n \left( \frac{\theta_i}{\sum_{j=1}^n \theta_j} \right)^{z_i} 
      & \sum_{j=1}^n z_j = k \\ 
      0 & \sum_{j=1}^n z_j \neq k
   \end{cases} \\
   = &
    f \left( \bm{z}; k, \frac{\theta_1}{\sum_{j=1}^n \theta_j}, \ldots, \frac{\theta_n}{\sum_{j=1}^n \theta_j} \right)
\end{align*}
 where $f \left( \bm{z}; k, \frac{\theta_1}{\sum_{j=1}^n \theta_j}, \ldots, \frac{\theta_n}{\sum_{j=1}^n \theta_j} \right)$ is the probability mass function of a multinomial distribution with parameter $k$ and $\frac{\theta_1}{\sum_{j=1}^n \theta_j}, \ldots, \frac{\theta_n}{\sum_{j=1}^n \theta_j}$.
\end{proof}

\subsection{Proposition~\ref{prop: Poisson Conditional Marginal}}
\label{proof:PoissonConditionalMarginal}

\begin{proposition}[Poisson Conditional Marginal and Expectations]\label{prop: Poisson Conditional Marginal}
    Given $\z = (z_1, \ldots, z_n)^T$ with $z_i \sim Poisson(\theta_i)$, the conditional marginal $p_{\logits}(z_i \mid \sum_{j=1}^n z_j = k)$ follows a binomial distribution with parameter $k$ and probability $\frac{\theta_i}{\sum_{j=1}^n \theta_j}$. Further, its expectation 
    is $\frac{k\theta_i}{\sum_{j=1}^n \theta_j}$.
\end{proposition}

\begin{proof}
Let $\z = (z_1, \ldots, z_n)^T$, where $z_i \sim Poisson(\theta_i)$. We compute a closed-form solution for the conditional marginal $p_{\logits} (z_i \mid \sum_{j=1}^n z_n = k)$.
Since the marginal of each variable of a multinomial distribution is a binomial distribution, then the conditional marginal is
\begin{align*}
    & p \left( z_i \mid \sum_{j=1}^n z_j = k \right) \\
    = &
    {k \choose z_i} \left( \frac{\theta_i}{\sum_{j=1}^n \theta_j} \right)^{z_i} \left( 1 - \frac{\theta_i}{\sum_{j=1}^n \theta_j} \right)^{n - z_i} 
\end{align*}
This is the probability mass function of a binomial distribution with parameter $k$ and probability $ \frac{\theta_i}{\sum_{j=1}^n \theta_j}$.
\end{proof}

\section{MOF Expected Loss NLL Explode}
In this section, we provide a mathematical explanation for the negative log likelihood to explode for closed-form expected loss in the MOF experiment. The L1 loss function is given by
\begin{align*}
    L(\btheta) 
    & =
    \sum_{i=1}^n
    \overline{\sigma}_i \sqrt{\frac{2}{\pi}} \exp \left( \frac{- (\overline{\mu}_i - y_i)^2}{2 \overline{\sigma}_i^2} \right) \\
    & +
    (\overline{\mu}_i - y_i) \,\textit{erf} \left( \frac{\overline{\mu}_i - y_i}{\sqrt{2 \overline{\sigma}_i^2}} \right)
\end{align*}
with $\overline{\mu}_i := \bm{\mu}_i + \frac{\bm{\Sigma}_{i,i}}{\sum_{t=1}^n \bm{\Sigma}_{t,t} } \left( k - \sum_{j=1}^n \bm{\mu}_j \right)$ and $\overline{\sigma}_i^2 := \bm{\Sigma}_{i,i} - \frac{\left( \bm{\Sigma}_{i,i} \right)^2}{\sum_{t=1}^n \bm{\Sigma}_{t,t} }$.
Define a constant $c$ such that $0 < c < 1$. Notice that if we scale the unconstrained variance $\bm{\Sigma}_{i,i}$ by $c$, $\overline{\mu}_{i, \text{scaled}} = \overline{\mu}_i$ and $\overline{\sigma}_{i, \text{scaled}}^2 = c \overline{\sigma}_i^2$.
\begin{align*}
    & \lim_{c \rightarrow 0} \overline{\sigma}_{i, \text{scaled}} \sqrt{\frac{2}{\pi}} \exp \left( \frac{- (\overline{\mu}_{i, \text{scaled}} - y_i)^2}{2 \overline{\sigma}_{i, \text{scaled}}^2} \right) \\
    = & \lim_{c \rightarrow 0} c \overline{\sigma}_{i} \sqrt{\frac{2}{\pi}} \exp \left( \frac{1}{c} \frac{- (\overline{\mu}_{i} - y_i)^2}{2 \overline{\sigma}_{i}^2} \right) \\
    = & 0
\end{align*}
Also,
\begin{align*}
    & \lim_{c \rightarrow 0} (\overline{\mu}_{i, \text{scaled}} - y_i) \,\textit{erf} \left( \frac{\overline{\mu}_{i, \text{scaled}} - y_i}{\sqrt{2 \overline{\sigma}_{i, \text{scaled}}^2}} \right) \\
    = & \lim_{c \rightarrow 0} (\overline{\mu}_{i} - y_i) \,\textit{erf} \left( \frac{1}{\sqrt{c}} \frac{\overline{\mu}_{i} - y_i}{\sqrt{2 \overline{\sigma}_{i}^2}} \right) \\
    = & | \overline{\mu}_{i} - y_i |
\end{align*}
As the scaling factor decreases, the expected loss converges, which shows that the expected loss favors variance with smaller magnitude. However, MAE is only associated with the constrained mean, which is invariant to scaling of the variance. Extremely small variance causes the constrained distribution to approach the shape of a dirac delta function, causing the NLL to explode.

\section{Hardware and Software Specification}
We implement our model in PyTorch. All experiments are run on servers/workstations with the following configuration:
\begin{itemize}
    \item 32 CPUs, 128G Mem, 4 × NVIDIA A5000 GPUs. Ubuntu 22.04.4
    \item 128 CPUs, 480G Mem, 8 × NVIDIA RTX 4090 GPUs. Ubuntu 22.04
    \item 48 CPUs, 200G Mem, 8 × NVIDIA V100 GPUs. Ubuntu 22.04
\end{itemize}

\end{document}